\RequirePackage[svgnames,table]{xcolor}
\documentclass[11pt,letterpaper]{yalearxiv}
\usepackage[numbers,sort&compress]{natbib}
\usepackage{multirow}
\usepackage{array}
\usepackage{subcaption}
\usepackage{algorithm}
\usepackage{algpseudocode}
\usepackage{adjustbox}
\usepackage{bookmark}

\makeatletter
\renewcommand\AB@affilsepx{\protect\\\protect\Affilfont}
\makeatother
\graphicspath{{figures/}}
\definecolor{splbg}{RGB}{235,242,255}
\newcolumntype{Q}{>{\columncolor{splbg}}c}

\theoremstyle{plain}
\newtheorem{proposition}{Proposition}
\newtheorem{theorem}{Theorem}
\newtheorem{lemma}{Lemma}

\theoremstyle{definition}
\newtheorem{assumption}{Assumption}

\fancypagestyle{firststyle}{%
  \fancyhf{}%
  \fancyhead[L]{\headerfont Preprint}%
  \fancyfoot[R]{\footerfont\thepage}%
}
\hypersetup{
  colorlinks=true,
  allcolors=YaleLightBlue,
  pdftitle={Cooperating with Future Collaborators: Multi-Agent RL under Staggered Participation},
  pdfauthor={Jianglin Qiao; Siyi Hu; Thien Hoang Nguyen; Zehong Cao; Salah Sukkarieh},
  pdfsubject={Preprint},
  pdfkeywords={Multi-Agent RL, Staggered Participation, Cooperative Tasks}
}

\newif\ifShowYaleLogo
\ShowYaleLogofalse

\newcommand{\YaleLogoHeight}{20pt}

\newcommand{\AdelaideLogoHeight}{18pt}
\newcommand{\CurtinLogoHeight}{18pt}
\newcommand{\SydneyLogoHeight}{20pt}

\newlength{\HeaderLogoBoxHeight}
\newcommand{\HeaderLogo}[2][]{%
  \raisebox{\dimexpr0.5\HeaderLogoBoxHeight-0.5\height\relax}%
    [\HeaderLogoBoxHeight][0pt]{%
      \includegraphics[#1]{#2}%
    }%
}

\renewcommand{\firstpagelogos}{%
  \makebox[\textwidth][s]{%

    \ifShowYaleLogo
      \HeaderLogo[
        height=\YaleLogoHeight,
        keepaspectratio
      ]{figures/yale_logo.png}%
      \hfill
    \fi
    \hfill

        \HeaderLogo[
      height=\SydneyLogoHeight,
      keepaspectratio
    ]{figures/logos/Sydney_logo.png}%
    
    \HeaderLogo[
      height=\AdelaideLogoHeight,
      keepaspectratio
    ]{figures/logos/Adelaide_logo.png}%
    
    \hfill
    \HeaderLogo[
      height=\CurtinLogoHeight,
      keepaspectratio
    ]{figures/logos/Curtin_logo.png}%

  }
}

\title{Cooperating with Future Collaborators:\\
Multi-Agent RL under Staggered Participation}
\runningtitle{Cooperating with Future Collaborators: Multi-Agent RL under Staggered Participation}

\DeclareRobustCommand{\EmailBreak}{\\[0.45em]}

\author[1]{Jianglin Qiao}
\author[2]{Siyi Hu}
\author[1]{Thien Hoang Nguyen}
\author[3]{Zehong Cao}
\author[1]{Salah Sukkarieh}

\affil[1]{ACFR, The University of Sydney, Sydney, Australia}
\affil[2]{School of EECMS, Curtin University, Bentley, Australia}
\affil[3]{School of CSIT, Adelaide University, Adelaide, Australia}

\vspace{0.3em}
\affil[ ]{%
\small
\texttt{\{jianglin.qiao,thienhoang.nguyen,salah.sukkarieh\}@sydney.edu.au}%
\EmailBreak
\texttt{siyi.hu@curtin.edu.au} \quad \texttt{jimmy.cao@adelaide.edu.au}%
\vspace{0.3em}
}
\begin{document}
\begin{abstract}
In cooperative Multi-Agent Reinforcement Learning (MARL), agents are often trained under concurrent participation, while in many tasks some agents act earlier and leave task-relevant information that becomes useful to agents participating later. We study this setting as \emph{staggered participation} (SP), which introduces a cross-time, cross-agent learning dependency because an early action may affect the return through the information it provides and the later policy that uses it. Learning under SP therefore requires both identifying what information is useful for future decisions and learning how later agents should use it. We propose \emph{Staggered Participation Learning} (SPL), a training-time augmentation that addresses these two parts with \emph{prospective acquisition supervision} for earlier agents and \emph{outcome-supervised receiver learning} for later agents. We evaluate SPL across multiple policy-based MARL backbones, environments, and staggered-participation patterns. Across 60 MPE/RWARE backbone setting comparisons, SPL achieves higher observed mean task completion in every case, with an average difference of 14.1\%. The gains also extend to eight-agent teams and a physics-based UAV--UGV environment in Isaac Lab, providing evidence across algorithmic, temporal, and embodied settings.
\end{abstract}
\maketitle
\section{Introduction}
\label{sec:intro}

Cooperative agents do not always participate over the same time interval. For instance, in multi-robot systems, one agent may begin sensing or exploration earlier, while another becomes active later for navigation or task execution~\citep{rockenbauer2024traversing,sasaki2020map,9732342}. Their participation windows may partially overlap, allowing direct interaction for some time steps, or they may be fully separated. In these cases, earlier agents can acquire or leave task-relevant information that remains useful to agents participating later. We study this temporal structure as \emph{staggered participation} (SP), where agents contribute to the same cooperative objective over temporally offset participation windows. As illustrated in Fig.~\ref{fig:problem}, these windows may partially overlap or be fully separated, while earlier participants can leave task-relevant information that supports later decisions.
\begin{figure}[t]
    \centering
    \includegraphics[width=1\linewidth]{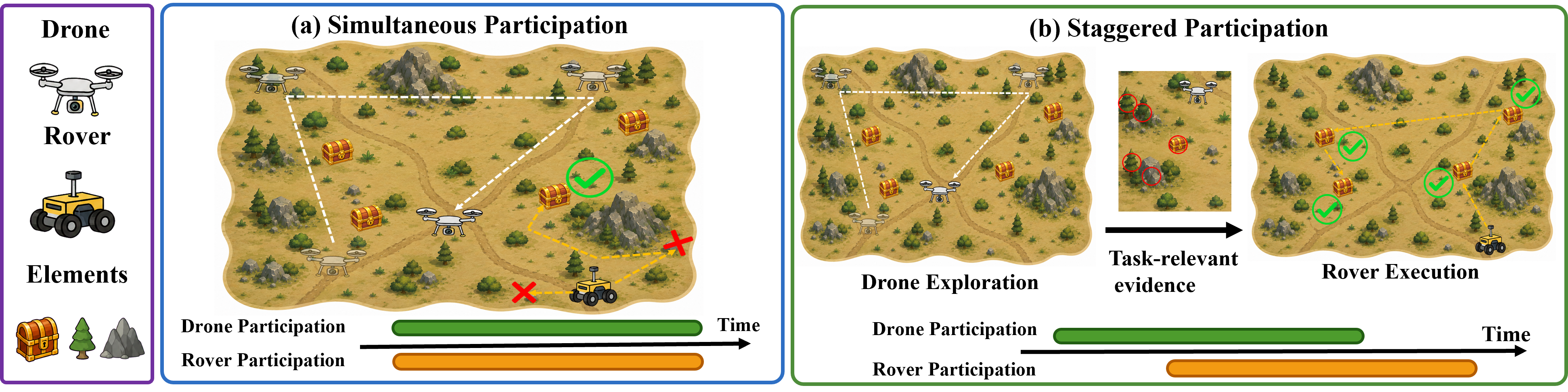}
    \caption{Simultaneous and staggered participation. In simultaneous participation, agents contribute within the same participation period. Under staggered participation, earlier participants acquire or leave information that remains available to future collaborators and supports later decisions. Participation windows may partially overlap or be fully separated.}
    \label{fig:problem}
\end{figure}

Learning under SP is challenging in cooperative MARL because the effect of an early action may only become visible through the behavior of an agent acting later. An earlier agent may take an action that acquires or reveals information for downstream decisions, but the usefulness of that action depends on how the resulting information affects later behavior. The return therefore links decisions across both time and participants. When participation windows overlap, this dependency can coexist with direct interaction; when they are disjoint, it is carried through information and other consequences of earlier actions. This gives rise to two interdependent learning requirements: \textit{(i) Information Acquisition}, where earlier agents learn which actions provide useful information for later decisions, and \textit{(ii) Downstream Utilization}, where later agents learn how that information should influence action selection. Although a shared team return couples the two roles through a common objective, it does not provide separate supervision for acquiring useful evidence and using it effectively downstream.

SP is closely related to three existing directions in cooperative MARL. {First, asynchronous MARL} studies agents that act or communicate at different times~\citep{3545946.3598752,min2023cooperative,3709347.3743915}. {Second, dynamic and open-team MARL} considers agents that join, leave, or interact with changing teammates~\citep{liu2021coach,rahman2023general,10.24963/ijcai.2024/974}. {Third, communication MARL} learns what information is useful for decision-making. Our focus is the learning dependency that appears when information acquisition and downstream utilization are distributed across different participation windows and potentially different agents. SP therefore places these two components within the same cooperative learning problem.

We introduce Staggered Participation Learning (SPL), a training-time augmentation that provides complementary supervision for the two learning requirements under SP. \emph{Prospective acquisition supervision} guides earlier agents toward actions that provide useful information for future collaborators, while \emph{outcome-supervised receiver learning} trains later agents to use inherited evidence based on downstream outcomes. SPL augments actor training without modifying environment rewards and introduces no additional communication or auxiliary inference at execution. We evaluate SPL across multiple policy-based MARL backbones, participation patterns, larger teams, and a physics-based UAV--UGV setting. Across the 60 primary Vanilla--SPL comparisons, SPL achieves higher observed mean task completion in every configuration, with an average gain of 14.1\%. Our contributions are: (i) \textit{Staggered participation.} We formulate cooperative MARL with temporally offset participation and identify the dependency between evidence acquisition and downstream utilization; (ii) \textit{Staggered Participation Learning.} We introduce complementary supervision for prospective acquisition and outcome-informed evidence use while preserving the underlying reward and decentralized execution; and (iii) \textit{Empirical evaluation.} We evaluate SPL across policy-based MARL backbones, temporal participation patterns, larger teams, and physics-based UAV--UGV simulation.

\section{Related Work}
\label{sec:related}
\paragraph{Simultaneous vs. Staggered Participation in Cooperative MARL.}
In simultaneous participation, agents contribute within the same interaction period. Cooperative MARL also studies temporal variation through asynchronous interaction, dynamic team composition, and staged decision making. Asynchronous methods consider agents that execute actions, communicate, or assign credit across asynchronous interactions~\citep{3545946.3598752,min2023cooperative,3709347.3743915,10.24963/ijcai.2025/20,farjadnasab2025cooperative}, while dynamic and open-team approaches address agents that join, leave, or interact with changing teammates~\citep{liu2021coach,rahman2023general,10.24963/ijcai.2024/974,hu2025adaptability}. Cooperative Consecutive Policies studies cooperation across consecutive decision stages~\citep{erskine2022developing}. These directions examine different aspects of when agents act and participate in a cooperative task. SP concerns temporally offset participation windows, which may partially overlap or be fully separated. In the tasks studied here, earlier participants can acquire task-relevant evidence before future collaborators begin participating, and this evidence remains available to support downstream decisions.

\paragraph{Direct vs. Indirect Communication Protocols in Cooperative MARL.}
Communication methods support direct information exchange and learn what to send and with whom to communicate, including under intermittent connectivity~\citep{zhu2024survey,hu2024learning,li2023context,yang2024team,lin2026and,ijcai2026p50}. Cooperation can also occur indirectly through persistent environmental traces or another agent's subsequent actions. Stigmergic methods support coordination through such environmental traces~\citep{xu2021stigmergic}, while MAE-MAPPO addresses indirect cooperation with invisible collaborators through marginal advantage estimation~\citep{pmlr-v337-qiao26a}. Related methods provide complementary mechanisms for retaining and interpreting information: sequence-based methods retain long-horizon context~\citep{mahjoub2025sable}, while SMPE$^2$ and MA$^2$E model or reconstruct information unavailable from local observations~\citep{kontogiannis2025enhancing,kang2025mae}. Policy distillation transfers training-time coordination knowledge into decentralized policies~\citep{pei2025policy,zhaidistilling}. Informative sensing and Dec-POMDP approaches further guide information acquisition according to its expected decision value~\citep{2484920.2484926,3306127.3331815,3709347.3743765}. The form of SP studied here combines temporally offset participation with indirect cooperation through persistent task-relevant evidence. Within this setting, SPL focuses on evidence acquisition and downstream utilization, while related credit-assignment methods study how cooperative outcomes are attributed across agents or coalitions~\citep{pmlr-v337-qiao26a,3709347.3743764}. SPL builds on value-aware information acquisition and policy distillation to supervise earlier participants and learn receiver preferences from attributable downstream outcomes. Persistent evidence connects the two roles without requiring additional execution-time communication.

\section{Methodology}\label{sec:method}
\subsection{Staggered Participation (SP)}\label{sec:staggered_participation}
\paragraph{Cooperative decision process.}
We model the cooperative task as a decentralized partially observable Markov decision process (Dec-POMDP)~\citep{oliehoek2016concise}, defined by
$\mathcal{M}=\langle \mathcal{N},\mathcal{S},\{\mathcal{O}_i\}_{i\in\mathcal{N}},\{\mathcal{U}_i\}_{i\in\mathcal{N}},P,R,\gamma\rangle$,
where $\mathcal{N}=\{1,\ldots,N\}$ is the set of agents and $\mathcal{S}$ is the global state space. At time $t$, the environment is in state $s_t\in\mathcal{S}$ and agent $i$ receives a local observation $o_t^i\in\mathcal{O}_i$. Each agent maintains an action--observation history $\tau_t^i$ and selects an action $u_t^i\in\mathcal{U}_i$ according to $\pi_{\theta_i}(u_t^i\mid\tau_t^i,M_t^i)$, where $M_t^i$ denotes the legal-action mask. The joint action $\mathbf{u}_t=(u_t^1,\ldots,u_t^N)$ produces the shared reward $r_t=R(s_t,\mathbf{u}_t)$ and the next state is sampled according to $s_{t+1}\sim P(\cdot\mid s_t,\mathbf{u}_t)$. All agents optimize the common discounted team return $G=\sum_{t=0}^{T-1}\gamma^t r_t$.

\paragraph{Staggered participation.}
Agents may participate over different temporal intervals, so teammates need not be active throughout the same period. Let $\mathcal{N}^{\mathrm{part}}_t\subseteq\mathcal{N}$ denote the set of agents participating at time $t$; agents outside this set take a null action and make no policy decision. For each agent $i$, let $[t_i^{\mathrm{start}},t_i^{\mathrm{end}})$ denote its participation interval, where $t_i^{\mathrm{start}}$ is the first participating time and $t_i^{\mathrm{end}}$ is the exclusive end time. We call agent $j$ a \emph{future collaborator} of agent $i$ if $t_i^{\mathrm{start}}<t_j^{\mathrm{start}}$ and both contribute to the same cooperative objective. Their participation intervals may be disjoint, with $t_i^{\mathrm{end}}\leq t_j^{\mathrm{start}}$, or may partially overlap, with $t_j^{\mathrm{start}}<t_i^{\mathrm{end}}$. In both cases, agent $i$ can act before agent $j$ begins participating.
\begin{figure*}[t]
    \centering
    \includegraphics[width=1\textwidth]{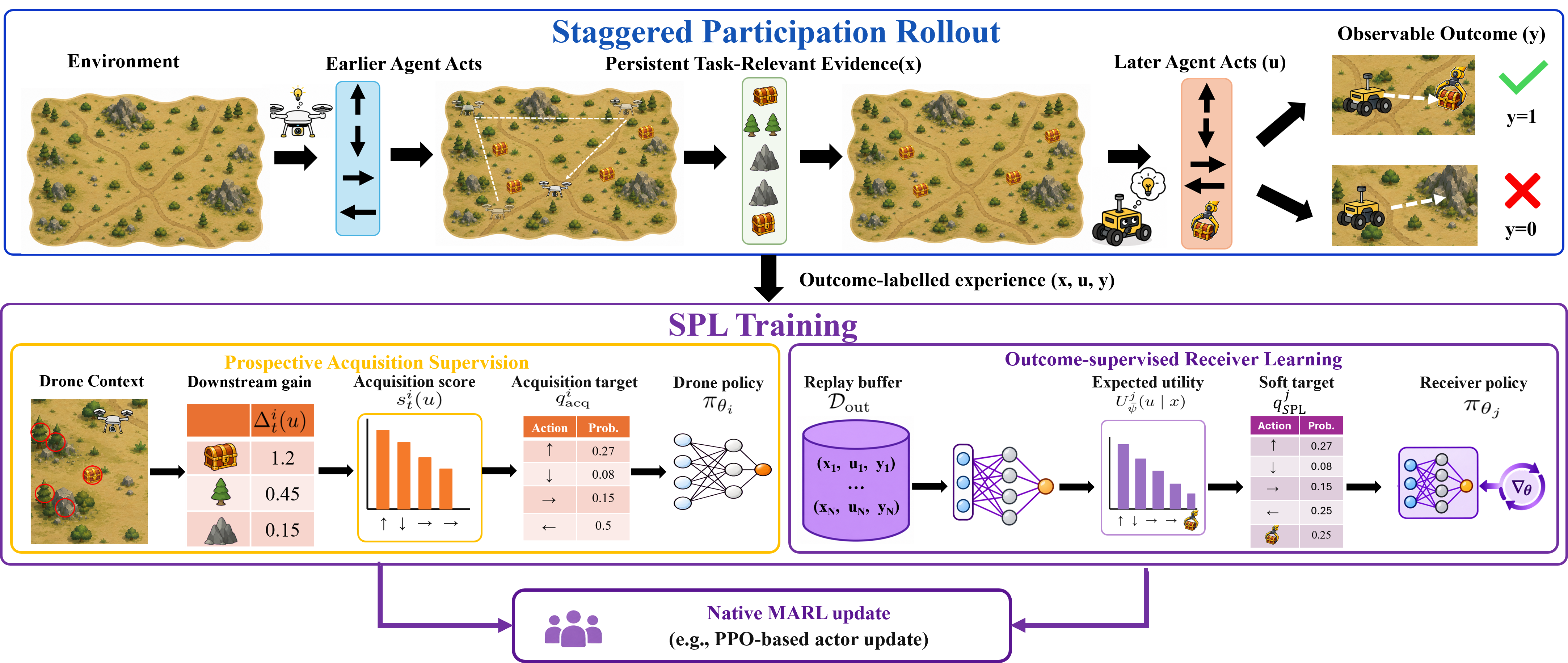}
    \caption{Overview of SP and SPL. Earlier participants receive prospective acquisition supervision, while future collaborators learn evidence-conditioned preferences from downstream outcomes. SPL is used only during training.}
    \label{fig:spl_pipeline}
\end{figure*}
\paragraph{Learning problem under SP}
Earlier information-gathering actions can change the evidence available to future collaborators and thereby influence downstream decisions and the shared return. This creates two interdependent learning requirements: what evidence earlier participants should acquire and how future collaborators should use it. Because the effect of an earlier action is realized through persistent evidence and subsequent decisions by another policy, its learning signal is delayed and indirect. We therefore focus on improving both evidence preparation and downstream evidence use. The next subsection introduces SPL to support these two requirements.

\subsection{Staggered Participation Learning (SPL)}
\label{sec:spl}

Under SP, earlier participants and future collaborators are coupled through persistent task-relevant evidence. SPL augments the native cooperative MARL objective with two complementary actor-side training signals. \emph{Prospective acquisition supervision} guides earlier participants toward evidence-gathering actions expected to support future decisions, while \emph{outcome-supervised receiver learning} teaches future collaborators how inherited evidence should influence downstream action selection. Both augment native actor training while leaving the environment reward and backbone critic objective unchanged, as illustrated in \textbf{Fig.~\ref{fig:spl_pipeline}}.

\paragraph{Structured training supervision.}
SPL uses structured task knowledge available during training to construct its auxiliary targets. Prospective acquisition supervision uses task-family validity priors, the fixed valid-task constraint, and the probe observation model to estimate the downstream value of candidate evidence-gathering actions. Receiver supervision uses fixed task-level utilities for successful and unsuccessful execution, while candidate-specific success probabilities are learned from attributable outcomes observed during training. These quantities describe task-generation, sensing, and outcome semantics; they do not reveal realized task validity, future outcomes, or a predefined correct action. In particular, the fixed acquisition model does not provide receiver success probabilities. The acquisition and receiver objectives are complementary but separately constructed: the acquisition target is fixed by the structured task model and is not updated from the learned receiver policy or outcome model. SPL therefore assumes that such task structure is available from the training environment or task specification. All SPL-specific models and targets remain training-only and are unavailable during execution. Implementation details are provided in \textbf{Appendix~B.4}.

\subsubsection{Prospective Acquisition Supervision}
\label{sec:spl_acquisition}

\paragraph{Prospective evidence value.}
Consider an eligible evidence-acquisition decision of an earlier participant $i$. Let $c_t^i$ denote its actor-visible acquisition context and $M_t^i$ the corresponding legal-action mask. For each legal task-directed action $u$, SPL uses a fixed training-time task model to estimate $\Delta_t^i(u)$, the expected increase in downstream decision value associated with the evidence produced by that action. SPL combines this prospective value with the corresponding travel or acquisition distance $d_t^i(u)$:
\begin{equation}
s_t^i(u)
=
\log\!\left(\max\{\Delta_t^i(u),\epsilon_{\mathrm{acq}}\}\right)
-
\frac{d_t^i(u)}{\tau_d},
\label{eq:spl_acquisition_score}
\end{equation}
where $\tau_d$ is a fixed distance temperature and $\epsilon_{\mathrm{acq}}>0$ is a numerical stabilizer. A deterministic transform $\mathcal{T}_{\mathrm{acq}}$ masks illegal actions, normalizes the scores of positive-gain task actions, and reserves a small fixed probability mass for the scan action, which gathers local evidence without targeting a specific candidate. If no positive-gain task action is available, the scan action receives all target mass. The resulting prospective acquisition target is $q_{\mathrm{acq}}^i=\mathcal{T}_{\mathrm{acq}}(s_t^i,M_t^i)$.

\paragraph{Acquisition policy supervision.}
For eligible earlier-participant decisions $\mathcal{D}_{\mathrm{acq}}$, SPL distils the prospective acquisition target into the decentralized policy:
\begin{equation}
\mathcal{L}_{\mathrm{SPL}}^{\mathrm{acq}}
=
\mathbb{E}_{(i,t)\sim\mathcal{D}_{\mathrm{acq}}}
\left[
\operatorname{CE}
\left(
q_{\mathrm{acq}}^i,
\pi_{\theta_i}(\cdot\mid\tau_t^i,M_t^i)
\right)
\right].
\label{eq:spl_acquisition_loss}
\end{equation}
This auxiliary loss updates only the earlier-participant actor.

\subsubsection{Outcome-Supervised Receiver Learning}
\label{sec:spl_receiver}

\paragraph{Outcome-labeled decisions.}
Consider a future collaborator $j$ making a downstream decision at time $t\geq t_j^{\mathrm{start}}$. Let $x_t^j$ denote the actor-visible information available before the decision, including retained task-relevant evidence and the collaborator's current local context. The collaborator selects $u_t^j$ using its decentralized policy, and the selected task subsequently produces an attributable outcome $y_t^j\in\{0,1\}$, where $y_t^j=1$ denotes successful completion and $y_t^j=0$ an attributable unsuccessful attempt. Each eligible decision contributes a tuple $(x_t^j,u_t^j,j,y_t^j)$ to the outcome replay set $\mathcal{D}_{\mathrm{out}}$. Subsequent observations are used only to determine the outcome label; we exclude interrupted, expired, or unresolved attempts when their outcomes cannot be reliably attributed.

\paragraph{Candidate outcome model.}
SPL learns how the evidence and local context available to a future collaborator relate to the outcome of each candidate action. Let $g_j(x,u)$ collect the actor-visible decision context, candidate-action features, and collaborator-role identity. A candidate-shared scorer $f_\psi$ predicts
\begin{equation}
p_{\psi}^{\,j}(u\mid x)
=
\operatorname{sigmoid}
\left(
f_{\psi}\!\left(g_j(x,u)\right)
\right),
\label{eq:spl_outcome_model}
\end{equation}
where $p_{\psi}^{\,j}(u\mid x)$ estimates the probability that candidate $u$ succeeds under information $x$. The model is trained from attributable outcomes using binary cross-entropy,
\begin{equation}
\mathcal{L}_{\mathrm{out}}(\psi)
=
\mathbb{E}_{(x,u,j,y)\sim\mathcal{D}_{\mathrm{out}}}
\left[
\ell_{\mathrm{BCE}}
\left(
y,p_{\psi}^{\,j}(u\mid x)
\right)
\right].
\label{eq:spl_outcome_loss}
\end{equation}
The outcome model is therefore learned from the observed consequences of the future collaborator's own selected actions. It receives no labels identifying which candidate action should have been selected.

\paragraph{Evidence-conditioned action utility.}
Outcome probability alone does not determine the desirability of a candidate when successful and unsuccessful executions have different task-level consequences. Let $c_j^{+}(u)$ and $c_j^{-}(u)$ denote fixed task-level utilities associated with successful and unsuccessful execution of candidate $u$, respectively. SPL combines these utilities with the evidence-conditioned success probability predicted by a slowly updated target outcome model with parameters $\bar{\psi}$:
\begin{equation}
U_{\bar{\psi}}^{\,j}(u\mid x)
=
p_{\bar{\psi}}^{\,j}(u\mid x)c_j^{+}(u)
+
\left[1-p_{\bar{\psi}}^{\,j}(u\mid x)\right]c_j^{-}(u).
\label{eq:spl_utility}
\end{equation}
Because $p_{\bar{\psi}}^{\,j}(u\mid x)$ is conditioned on decision-time information, the same candidate can receive different utility estimates under different evidence prepared by earlier participants.

\paragraph{Receiver target and policy supervision.}
Let $\mathcal{A}_t^j=\{u\in\mathcal{U}_j:M_t^j(u)=1\}$ denote the legal actions available to collaborator $j$. SPL standardizes the predicted utilities over this set and constructs a soft preference target:
\begin{equation}
z_t^j(u)
=
\frac{
U_{\bar{\psi}}^{\,j}(u\mid x_t^j)-\mu_t^j
}{
\sqrt{(\sigma_t^j)^2+\epsilon_z}
},
\qquad
q_{\mathrm{SPL}}^j(u)
=
\frac{
\exp(\beta z_t^j(u))
}{
\sum_{v\in\mathcal{A}_t^j}\exp(\beta z_t^j(v))
},
\label{eq:spl_receiver_target}
\end{equation}
where $\mu_t^j$ and $(\sigma_t^j)^2$ are the mean and variance of the legal-action utilities, $\epsilon_z>0$ is a variance stabilizer, and $\beta>0$ controls the strength of the induced preference. Illegal actions receive zero probability. Since the target is constructed from learned evidence-conditioned outcome predictions, it expresses relative preferences among the currently legal actions and does not require predefined correct-action labels. For eligible future-collaborator decision points $\mathcal{D}_{\mathrm{rec}}$, SPL distills this soft target into the decentralized policy:
\begin{equation}
\mathcal{L}_{\mathrm{SPL}}^{\mathrm{rec}}
=
\mathbb{E}_{(j,t)\sim\mathcal{D}_{\mathrm{rec}}}
\left[
\operatorname{CE}
\left(
q_{\mathrm{SPL}}^j,
\pi_{\theta_j}(\cdot\mid\tau_t^j,M_t^j)
\right)
\right].
\label{eq:spl_receiver_loss}
\end{equation}
The target is treated as fixed during this update. Gradients therefore modify only the future-collaborator actor and do not propagate through the outcome model, recorded outcomes, or environment dynamics.

\paragraph{Training and Execution.}
SPL combines $\mathcal{L}_{\mathrm{SPL}}^{\mathrm{acq}}$ for earlier participants with $\mathcal{L}_{\mathrm{SPL}}^{\mathrm{rec}}$ for future collaborators, while the MARL backbone continues to optimize its native objective from the shared team return. After each rollout, attributable outcomes are added to $\mathcal{D}_{\mathrm{out}}$ and used to update the outcome model once replay-readiness conditions are met; the target parameters are updated by $\bar{\psi}\leftarrow(1-\rho)\bar{\psi}+\rho\psi$. SPL actor updates are interleaved with native actor optimization, while critic optimization follows the original backbone procedure. During execution, agents act solely through $u_t^i\sim\pi_{\theta_i}(\cdot\mid\tau_t^i,M_t^i)$; no acquisition model, auxiliary target, outcome replay, outcome model, or utility estimate is required. We provide detailed readiness conditions, update ordering, gradient boundaries, and the complete training algorithm in \textbf{Appendices~B.6 and B.7}.

\subsection{Theoretical Analysis}
\label{sec:theoretical-analysis}

We provide a local analysis of the auxiliary objectives introduced by SPL. The results characterize the prospective acquisition target, calibration of the outcome-based receiver model, and local decision quality induced by the receiver target under the structured task assumptions of Sec.~\ref{sec:spl}. The acquisition result instantiates a standard value-of-information argument for the top-$k$ surrogate used by SPL, while the receiver results characterize calibration and distillation errors. Together, these results support the construction of the SPL training signals but do not form an end-to-end guarantee or establish global convergence of the cooperative return. Full proofs, assumptions, and guarantee boundaries are provided in \textbf{Appendix~A}.

\paragraph{Prospective acquisition.}
Consider task families represented by the fixed binary-validity model used to construct acquisition supervision. Let $b=(b_1,\ldots,b_n)$ denote posterior validity marginals under this model, where $b_r=\Pr(W_r=1\mid c_t^i)$. If a future collaborator can execute at most $k$ tasks, define the top-$k$ posterior completion value as
$V_k(b)=\max_{S:|S|\leq k}\sum_{r\in S}b_r$.
For a purely informational acquisition action $u$, let $Z_u$ denote the resulting observation and $b^{Z_u}$ the updated posterior. Its model-relative prospective gain is
$\Delta_t^i(u)=\mathbb{E}[V_k(b^{Z_u})\mid c_t^i,u]-V_k(b)$.
Let $\mathcal{P}_t^i=\{u:M_t^i(u)=1,\;u\text{ is task-directed},\;\Delta_t^i(u)>0\}$ denote the legal task-directed actions with positive prospective gain.

\begin{proposition}[Prospective acquisition target]
\label{prop:value-aware-acquisition}
Assume that acquisition actions change only the available information and that posterior updating and predictive expectation are Bayesian-consistent under the acquisition model. Then $\Delta_t^i(u)\geq0$ for every legal acquisition action. Moreover, before adding the scan mass, if $\mathcal{T}_{\mathrm{acq}}$ applies Gibbs normalization over $\mathcal{P}_t^i$, the acquisition target satisfies $q_{\mathrm{acq}}^i(u)\propto\exp(s_t^i(u))$ and uniquely maximizes the corresponding entropy-regularized score objective.
\end{proposition}

\paragraph{Proof sketch.}
The top-$k$ value $V_k$ is the maximum of linear functions of the posterior marginals and is therefore convex. Bayesian consistency gives
$\mathbb{E}[b^{Z_u}\mid c_t^i,u]=b$, so Jensen's inequality implies
$\mathbb{E}[V_k(b^{Z_u})\mid c_t^i,u]\geq V_k(b)$.
The Gibbs form follows from the standard entropy-regularized maximization problem. The full proof is provided in \textbf{Appendix~A.2}.

\paragraph{Outcome prediction and local utility.}
For an attributable future-collaborator decision, let $Y\in\{0,1\}$ denote its downstream outcome and let $C\in\{0,1\}$ indicate whether that outcome is retained in $\mathcal{D}_{\mathrm{out}}$. We define
$p_*^{\,j}(u\mid x)=\Pr(Y=1\mid X=x,U=u,J=j,C=1)$,
with corresponding local utility
$U_*^{\,j}(u\mid x)=p_*^{\,j}(u\mid x)c_j^+(u)+[1-p_*^{\,j}(u\mid x)]c_j^-(u)$.
For any predictor $p$, let $\mathcal{L}_{\mathrm{out}}(p)$ denote its population binary-cross-entropy risk over retained outcomes and $U_p^j$ the induced local utility.

\begin{proposition}[Retained-outcome calibration]
\label{prop:outcome-calibration}
On the covered retained-outcome support, $\mathcal{L}_{\mathrm{out}}(p)$ is uniquely minimized pointwise by $p_*^{\,j}$. If $C_\Delta=\sup_{u,j}|c_j^+(u)-c_j^-(u)|<\infty$, then
\begin{equation}
\mathbb{E}\!\left[
\left(
U_p^J(U\mid X)-U_*^J(U\mid X)
\right)^2
\mid C=1
\right]
\leq
\frac{C_\Delta^2}{2}
\left[
\mathcal{L}_{\mathrm{out}}(p)
-
\mathcal{L}_{\mathrm{out}}(p_*)
\right].
\label{eq:theory-utility-error}
\end{equation}
\end{proposition}

\paragraph{Proof sketch.}
Excess binary-cross-entropy risk equals the Bernoulli KL divergence on the retained-outcome distribution. Since
$U_p^j-U_*^j=(p^j-p_*^j)(c_j^+-c_j^-)$,
Pinsker's inequality gives Eq.~\ref{eq:theory-utility-error}. The result applies to retained, selected receiver actions and is conditional on $C=1$; an unconditional interpretation additionally requires $Y\perp C\mid X,U,J$. The full proof is provided in \textbf{Appendix~A.3}.

\paragraph{Receiver decision quality.}
Fix an eligible future-collaborator decision with legal action set $\mathcal{A}$ and $K=|\mathcal{A}|$. Let $z_*$ denote the reference standardized score and $z_{\bar{\psi}}$ the score produced by the target outcome model. Define
$\delta=\|z_{\bar{\psi}}-z_*\|_\infty$,
$R_z=\max_{u\in\mathcal{A}}z_*(u)-\min_{u\in\mathcal{A}}z_*(u)$, and
\[
\eta
=
\operatorname{KL}(q_{\mathrm{SPL}}^j\Vert\pi_t^j)
=
\operatorname{CE}(q_{\mathrm{SPL}}^j,\pi_t^j)
-
H(q_{\mathrm{SPL}}^j),
\]
where $q_{\mathrm{SPL}}^j(u)\propto\exp(\beta z_{\bar{\psi}}(u))$.
Thus, $\eta$ is directly measurable from the receiver target and actor distributions. The score error $\delta$ additionally depends on generalization across the full legal action set and score standardization; controlling it from outcome-model error requires stronger coverage conditions.

\begin{theorem}[Local receiver score regret]
\label{thm:local-receiver-regret}
For any legal receiver policy $\pi_t^j$ satisfying $\operatorname{KL}(q_{\mathrm{SPL}}^j\Vert\pi_t^j)<\infty$,
\begin{equation}
\operatorname{Reg}_{z,t}^j(\pi_t^j)
\leq
\frac{\log K}{\beta}
+
2\delta
+
R_z\sqrt{\frac{\eta}{2}},
\label{eq:theory-local-regret}
\end{equation}
where
$\operatorname{Reg}_{z,t}^j(\pi_t^j)
=
\max_{u\in\mathcal{A}}z_*(u)
-
\mathbb{E}_{u\sim\pi_t^j}[z_*(u)]$.
\end{theorem}

\paragraph{Proof sketch.}
The Boltzmann target contributes a softness gap of at most $\log K/\beta$. Replacing $z_{\bar{\psi}}$ with $z_*$ contributes at most $2\delta$, and Pinsker's inequality bounds the remaining distillation error by $R_z\sqrt{\eta/2}$. The three terms therefore capture finite-temperature softness, score approximation, and incomplete receiver distillation. The full proof is provided in \textbf{Appendix~A.5}.
\begin{figure}[t]
    \centering
    \includegraphics[width=1\linewidth]{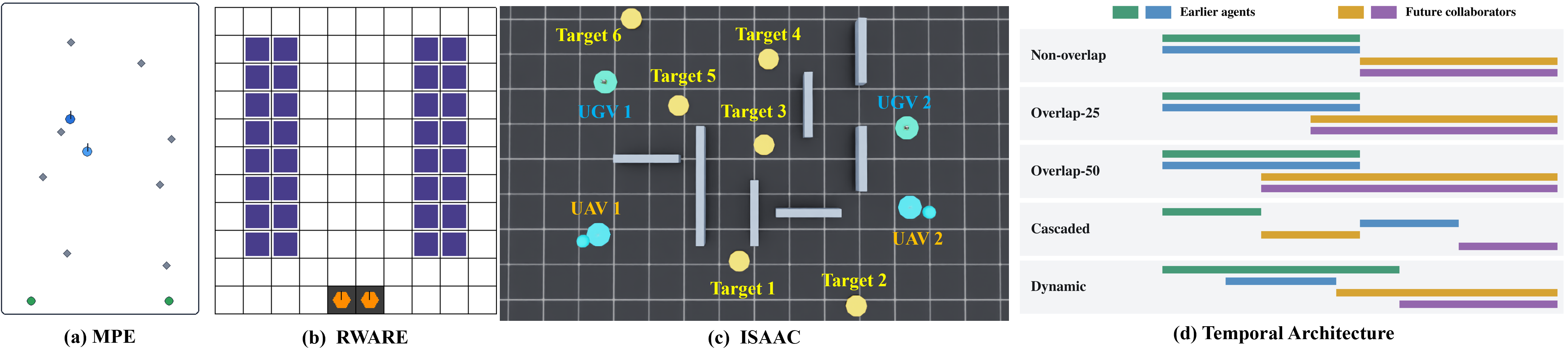}
    \caption{Evaluation environments and temporal participation patterns. (a) MPE, (b) RWARE, (c) the physics-based UAV--UGV environment in Isaac Lab, and (d) the five participation patterns used in MPE and RWARE.}
    \label{fig:env}
\end{figure}
\paragraph{Scope.}
The results characterize local components of SPL under the structured training assumptions of \textbf{Sec.~\ref{sec:spl}}. \textbf{Proposition~\ref{prop:value-aware-acquisition}} gives a model-relative information-value guarantee: internally consistent Bayesian updating preserves non-negative prospective gain, but misspecified priors can change how this gain relates to performance under the true environment distribution; this sensitivity is evaluated in \textbf{Appendix~D.6}. \textbf{Proposition~\ref{prop:outcome-calibration}} controls expected local-utility error on the retained selected-action distribution. \textbf{Theorem~\ref{thm:local-receiver-regret}} separates the measurable distillation error $\eta$ from the score-approximation term $\delta$ and bounds standardized local-score regret, not raw utility or full team return. Relating $\delta$ directly to outcome-model training requires stronger coverage and generalization assumptions across legal actions.

\begin{figure*}[t]
    \centering
    \includegraphics[width=1\linewidth]{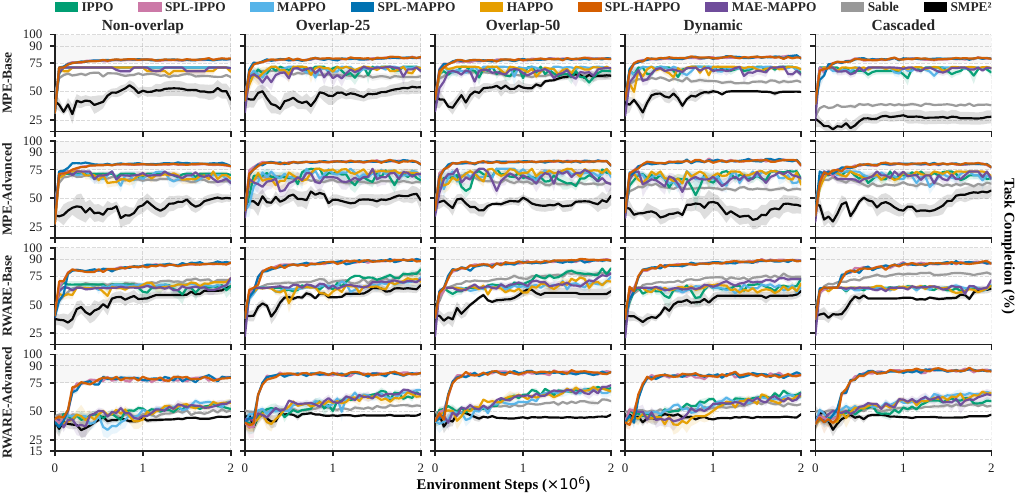}
    \caption{Training TC (\%) for all methods across 4 environments and 5 participation patterns averaged over 5 training seeds.}
    \label{fig:training-curves}
\end{figure*}

\section{Experimental Setup}\label{sec:experiments}
\paragraph{Environments.} We evaluate SPL in MPE, RWARE, and a physics-based UAV--UGV environment, where earlier participants gather task-relevant evidence before future collaborators perform downstream tasks. MPE-Base builds on the Multi-Agent Particle Environment~\citep{lowe2017multi}, while MPE-Advanced further randomizes task assignments, spawn positions, and future-collaborator identities. RWARE-Base transfers the same evidence-mediated cooperation problem to the Multi-Robot Warehouse domain~\citep{papoudakis2021benchmarking}, and RWARE-Advanced additionally introduces task permutation, a shared candidate pool, and dynamic task claims. Isaac Lab~\citep{mittal2025isaac} provides an embodied UAV--UGV setting with physical motion and obstacle constraints. Full environment specifications are given in \textbf{Appendices~B.1--B.3}. For MPE and RWARE, we evaluate five patterns: \emph{Non-overlap}, with a strict handoff between roles; \emph{Overlap-25} and \emph{Overlap-50}, with fixed temporal overlap; \emph{Dynamic}, with stochastic pair-specific handoff times; and \emph{Cascaded}, with successive participation intervals. Isaac Lab uses corresponding UAV--UGV schedules. Exact participation windows are provided in \textbf{Appendices~B.2--B.3}.

\paragraph{Implementation} We augment three PPO-based MARL backbones---IPPO, MAPPO, and HAPPO---to obtain SPL-IPPO, SPL-MAPPO, and SPL-HAPPO. MAPPO~\citep{yu2022surprising} is the primary backbone, while HAPPO~\citep{kuba2022trust} and IPPO~\citep{de2020independent} test compatibility with heterogeneous and independent learning, respectively. We additionally compare with MAE-MAPPO~\citep{pmlr-v337-qiao26a}, Sable~\citep{mahjoub2025sable}, and SMPE$^2$~\citep{kontogiannis2025enhancing} as complementary cross-method references. The primary controlled comparisons remain between each Vanilla backbone and its SPL-augmented counterpart. Baseline adaptations and implementation details are provided in \textbf{Appendix~B.5}. Our implementation builds on HARL~\citep{JMLR:v25:23-0488}. SPL preserves the native actor--critic architecture and training objective of each backbone and adds only role-specific auxiliary actor updates during training; SPL-specific models are not required during execution. Optimization details, training procedures, and observed training costs are reported in \textbf{Appendices~B--C}.
\paragraph{Evaluation.}
Task completion (TC) is the primary metric. The four-agent MPE/RWARE experiments use five independent training seeds per method and configuration, while the eight-agent and Isaac Lab evaluations use three; all final evaluations contain 200 episodes per seed. Results are reported as mean $\pm$ standard error across training seeds. Throughout the results, reported TC gains denote absolute differences between percentage-valued TC scores and are written using \%. The primary Vanilla--SPL comparisons are additionally examined using paired seed-level analyses where same-seed pairing is available; complete confidence intervals, Holm-adjusted tests, and pairing details are reported in \textbf{Appendix~D.7}. Additional metrics, including full success, return, failed execution, and exploration diagnostics, are reported in \textbf{Appendix~D}. All configurations are trained and evaluated independently, so the reported results characterize within-configuration performance rather than transfer across participation patterns or task structures.

\begin{table*}[t]
\centering
\caption{Final task completion (TC, \%) after 2M environment steps. Values are mean $\pm$ standard error over five training seeds. SPL columns are shaded; bold indicates the highest observed mean in each setting and does not denote statistical significance.}
\label{tab:s4-task-completion}

\scriptsize
\setlength{\tabcolsep}{1.8pt}
\renewcommand{\arraystretch}{0.98}

\resizebox{\textwidth}{!}{%
\begin{tabular}{@{}ll
c>{\columncolor{splbg}}c
c>{\columncolor{splbg}}c
c>{\columncolor{splbg}}c
ccc@{}}
\toprule
\multirow{2}{*}{Environment}
& \multirow{2}{*}{Participation}
& \multicolumn{2}{c}{IPPO}
& \multicolumn{2}{c}{MAPPO}
& \multicolumn{2}{c}{HAPPO}
& \multicolumn{3}{c}{Other baselines} \\
\cmidrule(lr){3-4}
\cmidrule(lr){5-6}
\cmidrule(lr){7-8}
\cmidrule(lr){9-11}
&
& Van. & \cellcolor{splbg}SPL
& Van. & \cellcolor{splbg}SPL
& Van. & \cellcolor{splbg}SPL
& MAE-MAPPO & Sable & SMPE$^2$ \\
\midrule

\multirow{5}{*}{MPE-Base}
& Non-overlap
& $70.50 \pm 0.00$
& $78.38 \pm 0.38$
& $70.50 \pm 0.00$
& $78.73 \pm 0.43$
& $67.45 \pm 3.05$
& $\mathbf{79.18 \pm 0.23}$
& $70.50 \pm 0.00$
& $62.70 \pm 0.25$
& $42.73 \pm 9.29$ \\

& Overlap-25
& $68.25 \pm 2.75$
& $79.35 \pm 0.57$
& $70.78 \pm 0.14$
& $80.25 \pm 0.21$
& $70.70 \pm 0.18$
& $\mathbf{80.35 \pm 0.29}$
& $67.93 \pm 2.95$
& $63.50 \pm 2.24$
& $53.85 \pm 5.43$ \\

& Overlap-50
& $67.85 \pm 3.15$
& $\mathbf{78.90 \pm 0.31}$
& $69.53 \pm 1.27$
& $\mathbf{78.90 \pm 0.24}$
& $69.80 \pm 0.92$
& $78.50 \pm 0.47$
& $70.63 \pm 0.23$
& $64.43 \pm 1.85$
& $63.95 \pm 3.34$ \\

& Dynamic
& $66.20 \pm 3.01$
& $\mathbf{79.65 \pm 0.18}$
& $70.23 \pm 0.59$
& $78.88 \pm 0.18$
& $69.53 \pm 0.99$
& $79.63 \pm 0.39$
& $64.85 \pm 3.06$
& $58.30 \pm 1.11$
& $49.70 \pm 1.04$ \\

& Cascaded
& $66.83 \pm 3.68$
& $78.90 \pm 0.70$
& $70.63 \pm 0.10$
& $\mathbf{78.93 \pm 0.48}$
& $70.68 \pm 0.19$
& $78.78 \pm 0.22$
& $70.45 \pm 0.05$
& $38.03 \pm 0.93$
& $27.93 \pm 5.74$ \\

\midrule

\multirow{5}{*}{MPE-Advanced}
& Non-overlap
& $70.38 \pm 0.00$
& $77.95 \pm 0.19$
& $64.08 \pm 3.86$
& $\mathbf{78.23 \pm 0.16}$
& $67.28 \pm 3.10$
& $78.03 \pm 0.16$
& $63.28 \pm 3.47$
& $64.90 \pm 1.36$
& $49.78 \pm 3.25$ \\

& Overlap-25
& $65.65 \pm 3.53$
& $\mathbf{79.95 \pm 0.52}$
& $68.28 \pm 1.50$
& $79.38 \pm 0.27$
& $70.93 \pm 0.56$
& $79.55 \pm 0.40$
& $71.40 \pm 0.23$
& $64.30 \pm 1.63$
& $47.83 \pm 9.02$ \\

& Overlap-50
& $71.30 \pm 0.22$
& $78.63 \pm 0.23$
& $68.58 \pm 3.05$
& $78.33 \pm 0.40$
& $71.40 \pm 0.14$
& $\mathbf{78.68 \pm 0.25}$
& $62.63 \pm 5.92$
& $61.88 \pm 1.48$
& $52.08 \pm 8.25$ \\

& Dynamic
& $66.63 \pm 3.66$
& $78.23 \pm 0.29$
& $67.05 \pm 3.20$
& $\mathbf{79.08 \pm 0.54}$
& $61.73 \pm 5.55$
& $78.80 \pm 0.18$
& $67.68 \pm 1.76$
& $57.60 \pm 1.30$
& $43.10 \pm 6.61$ \\

& Cascaded
& $67.68 \pm 2.58$
& $76.80 \pm 0.13$
& $67.28 \pm 3.10$
& $77.00 \pm 0.33$
& $69.28 \pm 1.10$
& $\mathbf{77.23 \pm 0.18}$
& $66.83 \pm 3.02$
& $62.73 \pm 0.98$
& $56.70 \pm 5.50$ \\

\midrule

\multirow{5}{*}{RWARE-Base}
& Non-overlap
& $66.45 \pm 5.86$
& $86.83 \pm 0.40$
& $71.38 \pm 2.27$
& $86.70 \pm 0.49$
& $73.18 \pm 1.50$
& $\mathbf{86.93 \pm 0.26}$
& $73.48 \pm 1.69$
& $71.80 \pm 0.98$
& $66.05 \pm 3.29$ \\

& Overlap-25
& $81.10 \pm 3.41$
& $\mathbf{89.15 \pm 0.16}$
& $74.25 \pm 3.41$
& $89.08 \pm 0.35$
& $70.15 \pm 2.93$
& $88.35 \pm 0.50$
& $72.50 \pm 2.86$
& $76.70 \pm 1.18$
& $67.28 \pm 2.87$ \\

& Overlap-50
& $82.00 \pm 1.12$
& $\mathbf{89.00 \pm 0.19}$
& $74.73 \pm 3.35$
& $88.58 \pm 0.28$
& $70.28 \pm 2.41$
& $88.80 \pm 0.23$
& $77.53 \pm 0.72$
& $77.23 \pm 1.97$
& $61.95 \pm 4.09$ \\

& Dynamic
& $72.88 \pm 1.19$
& $\mathbf{88.88 \pm 0.27}$
& $65.35 \pm 5.46$
& $88.70 \pm 0.30$
& $68.80 \pm 4.06$
& $88.40 \pm 0.42$
& $72.28 \pm 1.51$
& $74.63 \pm 1.22$
& $62.03 \pm 4.29$ \\

& Cascaded
& $66.80 \pm 1.98$
& $86.25 \pm 0.38$
& $67.53 \pm 2.05$
& $\mathbf{86.75 \pm 0.50}$
& $66.58 \pm 2.24$
& $85.95 \pm 0.79$
& $71.63 \pm 0.73$
& $76.83 \pm 0.93$
& $64.18 \pm 2.94$ \\

\midrule

\multirow{5}{*}{RWARE-Advanced}
& Non-overlap
& $52.08 \pm 1.37$
& $79.93 \pm 0.59$
& $58.63 \pm 2.31$
& $\mathbf{80.03 \pm 0.23}$
& $57.58 \pm 3.44$
& $79.55 \pm 0.79$
& $58.50 \pm 2.20$
& $50.70 \pm 0.32$
& $45.70 \pm 0.25$ \\

& Overlap-25
& $64.18 \pm 1.62$
& $83.23 \pm 0.47$
& $68.83 \pm 1.40$
& $83.65 \pm 0.32$
& $63.28 \pm 1.98$
& $\mathbf{83.73 \pm 0.48}$
& $64.70 \pm 1.81$
& $54.75 \pm 0.44$
& $47.68 \pm 0.66$ \\

& Overlap-50
& $70.35 \pm 1.28$
& $\mathbf{84.55 \pm 0.69}$
& $68.05 \pm 4.14$
& $84.45 \pm 0.33$
& $67.50 \pm 2.12$
& $84.43 \pm 0.56$
& $72.88 \pm 1.05$
& $58.88 \pm 0.77$
& $47.05 \pm 0.38$ \\

& Dynamic
& $66.60 \pm 1.95$
& $81.30 \pm 0.34$
& $62.55 \pm 2.83$
& $\mathbf{82.50 \pm 0.43}$
& $63.23 \pm 0.83$
& $82.20 \pm 0.16$
& $65.90 \pm 2.33$
& $56.30 \pm 0.49$
& $47.55 \pm 0.40$ \\

& Cascaded
& $58.95 \pm 1.70$
& $85.40 \pm 0.42$
& $65.55 \pm 1.76$
& $\mathbf{85.98 \pm 0.44}$
& $66.03 \pm 2.06$
& $84.90 \pm 0.51$
& $64.18 \pm 3.98$
& $54.90 \pm 1.08$
& $47.03 \pm 0.55$ \\

\bottomrule
\end{tabular}%
}
\end{table*}

\section{Results}
\subsection{Main Results}
\label{sec:main-results}

\textbf{Fig.~\ref{fig:training-curves}} shows the training task-completion trajectories across the four environments and five participation patterns. In most settings, the SPL variants reach a higher task-completion regime than their corresponding Vanilla backbones and maintain this difference through 2M environment steps. Although the Vanilla IPPO, MAPPO, and HAPPO trajectories differ, their SPL counterparts often converge to similar final performance, suggesting that the observed improvement is not specific to a single PPO-based actor--critic backbone. \textbf{Table~\ref{tab:s4-task-completion}} reports the final results for the 60 primary Vanilla--SPL comparisons. SPL achieves higher observed mean TC in every configuration, with an equal-weight descriptive mean difference of 14.09\%. The corresponding mean differences are 13.93\%, 14.02\%, and 14.33\% for IPPO, MAPPO, and HAPPO, respectively. Among the 57 configurations with paired inference, 56 pointwise 95\% confidence intervals have positive lower bounds, and 29 remain significant after Holm correction over the declared family of 60 comparisons at a familywise level of 0.05. Complete statistical results are reported in \textbf{Appendix~D.7}, while full TC, FS, return, and FE results are provided in \textbf{Appendices~D.1--D.4}.

Participation structure affects the magnitude of the observed improvement. Averaged across environments and backbones, SPL improves task completion by 15.1\% under Non-overlap, 13.3\% under Overlap-25, 11.7\% under Overlap-50, 15.5\% under Dynamic, and 14.9\% under Cascaded participation. Across the fixed Non-overlap, Overlap-25, and Overlap-50 schedules, the average gain decreases as cross-role concurrency increases. Since the exact temporal constructions differ between MPE and RWARE, we treat this pattern as descriptive and do not interpret it as a controlled effect of overlap. The positive gains under both overlap settings nevertheless show that SPL does not require a strictly sequential handoff. Improvements under Dynamic and Cascaded participation further show that the observed benefit extends to stochastic and successively staged participation schedules. We additionally evaluate SPL-MAPPO in eight-agent MPE-Advanced and RWARE-Advanced settings. As shown in \textbf{Table~\ref{tab:additional-tc}}, SPL-MAPPO achieves higher mean task completion than MAPPO in all six matched comparisons, with an average observed gain of 17.1\%. The gains range from 3.3\% in MPE-Advanced Cascaded to more than 30\% in the Non-overlap and Dynamic RWARE-Advanced settings. Non-overlap and Dynamic use 2M environment steps, whereas Cascaded uses 4M to accommodate its longer participation horizon; differences across participation patterns should therefore not be interpreted as controlled comparisons of training budget. SPL introduces no additional models during execution. Its training-time computational cost, however, varies across environments and team sizes: the overhead is small in several four-agent settings but becomes more pronounced in some larger-scale configurations. Detailed wall-clock measurements are reported in \textbf{Appendix~C}.

\begin{table}[t]
\centering
\caption{Task completion (\%) for eight-agent scaling and physics-based UAV--UGV evaluations using MAPPO.}
\label{tab:additional-tc}
\footnotesize
\setlength{\tabcolsep}{3.2pt}
\renewcommand{\arraystretch}{1.02}
\begin{tabular}{@{}l c Q c Q@{}}
\toprule
& \multicolumn{2}{c}{MPE-Advanced (8)}
& \multicolumn{2}{c}{RWARE-Advanced (8)} \\
\cmidrule(lr){2-3}
\cmidrule(lr){4-5}
Participation
& Vanilla & \cellcolor{splbg}SPL
& Vanilla & \cellcolor{splbg}SPL \\
\midrule
Non-overlap
& $64.21\pm0.61$
& $\mathbf{77.77\pm0.30}$
& $41.10\pm2.66$
& $\mathbf{74.38\pm1.24}$ \\

Dynamic
& $68.60\pm2.64$
& $\mathbf{77.75\pm0.20}$
& $44.81\pm1.57$
& $\mathbf{76.15\pm0.88}$ \\

Cascaded
& $71.79\pm0.27$
& $\mathbf{75.13\pm0.14}$
& $19.73\pm0.70$
& $\mathbf{31.75\pm1.06}$ \\

\midrule

& \multicolumn{2}{c}{1 UAV + 1 UGV}
& \multicolumn{2}{c}{2 UAVs + 2 UGVs} \\
\cmidrule(lr){2-3}
\cmidrule(lr){4-5}
Participation
& Vanilla & \cellcolor{splbg}SPL
& Vanilla & \cellcolor{splbg}SPL \\
\midrule
Non-overlap
& $31.50\pm5.97$
& $\mathbf{61.58\pm0.67}$
& $55.79\pm3.23$
& $\mathbf{76.38\pm1.52}$ \\

Dynamic
& $36.67\pm0.71$
& $\mathbf{59.58\pm1.79}$
& $58.46\pm3.40$
& $\mathbf{79.25\pm1.30}$ \\

\bottomrule
\end{tabular}
\end{table}

\subsection{Physics-Based UAV--UGV Evaluation}
\label{sec:isaac}

We further evaluate SPL in a physics-based UAV--UGV setting in Isaac Lab, where earlier UAVs navigate to sensing locations, and later UGVs navigate to and execute downstream tasks. This introduces embodied motion and obstacle constraints on both evidence acquisition and downstream execution. As shown in \textbf{Table~\ref{tab:additional-tc}}, SPL-MAPPO achieves higher mean task completion than MAPPO in all four team-size--participation settings. The gains range from 20.6\% to 30.1\% across both team sizes and occur under both Non-overlap and Dynamic participation. Attempt-level diagnostics show more resolved executions and successful completions with SPL-MAPPO, together with a lower observed failure rate per resolved attempt, while map and published evidence coverage vary with team size. These results provide complementary evidence that the observed SPL improvement persists when evidence acquisition and downstream task execution are coupled to embodied motion. Execution, coverage, and safety diagnostics are reported in \textbf{Appendix~D.5}.

\begin{figure}[t]
    \centering
    \includegraphics[width=0.5\linewidth]{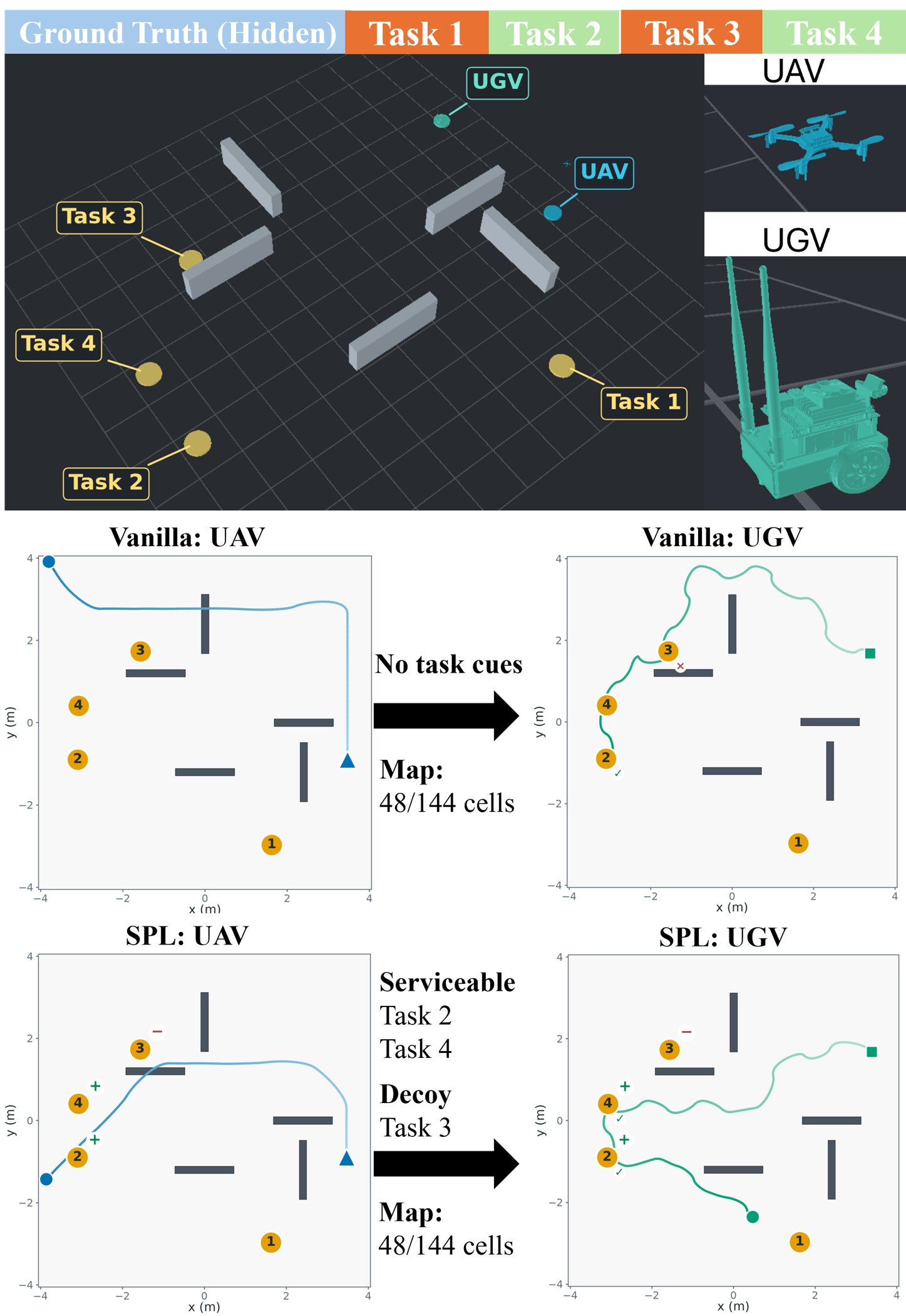}
    \caption{Isaac Lab case study comparing Vanilla and SPL based on MAPPO.}
    \label{fig:isaac-case}
\end{figure}

\paragraph{Qualitative case study.}
Figure~\ref{fig:isaac-case} compares MAPPO and SPL-MAPPO on the same deterministic fixed-handoff rollout in the one-UAV--one-UGV Isaac setting. The hidden ground truth, shown only for interpretation, contains two decoy tasks (Tasks~1 and~3) and two serviceable tasks (Tasks~2 and~4). All candidate locations are visible to the agents, while task serviceability must be inferred from acquired evidence. The UAV acts during the first 200 steps of the 600-step episode, after which the UGV acts using the persistent shared board. In the trajectory panels, triangles and squares mark UAV and UGV starting positions, circles mark trajectory endpoints, plus/minus symbols denote published task cues, and check/cross symbols denote UGV service outcomes; trajectory color darkens with time. At handoff, both methods have mapped 48 of 144 occupancy cells, but their task-specific evidence differs substantially. MAPPO publishes no task-quality cues, after which its UGV attempts decoy Task~3 at step~453 and completes serviceable Task~2 at step~577, leaving Task~4 unserved. SPL-MAPPO instead acquires cues for Tasks~2, 3, and~4, which are published to the shared board after the five-step delay: positive for Tasks~2 and~4 and negative for Task~3. Its UGV then completes Task~4 at step~438 and Task~2 at step~484 without attempting a decoy. This example makes the cross-stage information pathway explicit: similar spatial coverage can coexist with different task-specific evidence and, consequently, different downstream decisions.

\subsection{Ablation Study}
\label{sec:ablation}
We evaluate SPL-MAPPO under Non-overlap participation to examine the contribution of its two training components, the role of inherited evidence in receiver learning, and the dependence on structured training information. \emph{Acquisition-only} and \emph{Receiver-only} retain only the corresponding SPL component. \emph{Evidence-redacted} removes inherited task evidence consistently from both outcome-model inputs and receiver-target construction, whereas \emph{SKD-MAPPO} retains full evidence during outcome-model training but masks predecessor-generated evidence when constructing receiver targets. \emph{Outcome-agnostic} fixes candidate success probabilities to $0.5$, while \emph{Permuted} assigns predicted utilities to incorrect legal actions while preserving the receiver-update budget. \textbf{Table~\ref{tab:spl-ablation}} shows that full SPL achieves the highest observed mean task completion in all four environments. Relative to full SPL, Acquisition-only and Receiver-only reduce mean task completion by 10.7\% and 7.7\% on average, respectively. Neither component alone reproduces the complete method, supporting the complementary roles of evidence preparation and downstream evidence use.
\begin{table}[t]
\centering
\caption{SPL ablations under Non-overlap.}
\label{tab:spl-ablation}
\footnotesize
\setlength{\tabcolsep}{2.5pt}
\renewcommand{\arraystretch}{1.05}
\begin{tabular}{lcccc}
\toprule
Variant & MPE-Base & MPE-Adv. & RWARE-Base & RWARE-Adv. \\
\midrule
Full SPL
& $\mathbf{78.73\pm0.43}$
& $\mathbf{78.23\pm0.16}$
& $\mathbf{86.70\pm0.49}$
& $\mathbf{80.03\pm0.23}$ \\

SKD-MAPPO
& $70.50\pm0.00$
& $70.38\pm0.00$
& $70.75\pm0.00$
& $67.63\pm0.23$ \\

\midrule

Acquisition-only
& $70.50\pm0.00$
& $69.88\pm0.50$
& $77.80\pm2.37$
& $62.83\pm1.98$ \\

Receiver-only
& $73.45\pm1.29$
& $72.28\pm0.94$
& $74.40\pm1.94$
& $72.90\pm1.03$ \\

Evidence-redacted
& $70.50\pm0.00$
& $70.38\pm0.00$
& $70.75\pm0.00$
& $69.13\pm0.21$ \\

Outcome-agnostic
& $70.50\pm0.00$
& $70.38\pm0.00$
& $62.33\pm2.27$
& $48.48\pm0.80$ \\

Permuted
& $37.13\pm0.37$
& $37.05\pm0.87$
& $26.80\pm2.76$
& $28.95\pm0.84$ \\
\bottomrule
\end{tabular}
\end{table}

SKD-MAPPO also remains below full SPL in every environment, with an average gap of 11.1\%. Thus, matching SPL's training-side structured information and auxiliary machinery without evidence-conditioned receiver supervision does not reproduce the full SPL performance. We further test sensitivity to misspecified validity priors while keeping the true environment distribution unchanged. Even with fully reversed priors ($\lambda=1$), SPL-MAPPO retains $76.73\%$ and $77.45\%$ TC in MPE-Advanced and RWARE-Advanced, compared with nominal values of $78.23\%$ and $80.03\%$, respectively. Full success is more sensitive, particularly in RWARE-Advanced; complete results are reported in \textbf{Appendix~D.6}. These results show that SPL retains substantial task-completion performance under validity-prior misspecification in the tested settings. Other forms of structured-model error remain outside the scope of this experiment. Receiver-target ablations further show that the information used to construct the target matters. Evidence-redacted, Outcome-agnostic, and Permuted reduce mean task completion by 10.7\%, 18.0\%, and 48.4\% on average, respectively. The strong degradation under Permuted, despite preserving the receiver-update budget, shows that the correspondence between candidate actions and their predicted utilities is important. Preserving the auxiliary update budget alone is insufficient to recover the full SPL performance.

\section{Conclusion}
We studied cooperative MARL under staggered participation and introduced SPL, which combines prospective evidence acquisition with downstream evidence use. SPL achieves higher observed mean task completion across multiple policy-based backbones, participation patterns, larger teams, and physics-based UAV--UGV simulation. These results support explicitly addressing both information acquisition and downstream utilization when agents contribute over staggered participation windows.

\clearpage
\begingroup
\small
\setlength{\bibsep}{5pt}
\bibliographystyle{unsrtnat}
\bibliography{main}
\endgroup

\clearpage
\appendix
\numberwithin{equation}{section}
\counterwithin{table}{section}
\counterwithin{figure}{section}
\section{Proofs and Additional Analysis for SPL}
\label{app:spl-theory}

This section provides full proofs for Proposition~1, Proposition~2, and
Theorem~1 in the main paper, together with supporting lemmas and additional
analysis of the SPL auxiliary objectives. The acquisition result applies a
standard value-of-information argument to the top-$k$ completion surrogate
used by SPL. For the receiver analysis, we additionally relate the
distillation term directly to the training objective and give a conditional
connection between outcome-model error and standardized-score error.
Section~\ref{app:theory-guarantee-boundary} summarizes the scope of the
analysis.

\subsection{Notation and Assumptions}
\label{app:theory-setting}

Agent $i$ denotes an earlier participant and agent $j$ a future collaborator.
The actor-visible acquisition context of $i$ is denoted by $c_t^i$, while
$x_t^j$ denotes the actor-visible receiver context of $j$. The legal receiver
action set is
\[
\mathcal{A}_t^j
=
\{u\in\mathcal{U}_j:M_t^j(u)=1\}.
\]

For an attributable receiver decision, let $Y\in\{0,1\}$ denote its downstream
outcome and let $C\in\{0,1\}$ indicate whether that outcome is retained in
$\mathcal{D}_{\mathrm{out}}$. The retained-outcome probability is
\begin{equation}
p_*^{\,j}(u\mid x)
=
\Pr(Y=1\mid X=x,U=u,J=j,C=1).
\label{eq:theory-retained-outcome-probability}
\end{equation}
The corresponding local utility is
\begin{equation}
U_*^{\,j}(u\mid x)
=
p_*^{\,j}(u\mid x)c_j^+(u)
+
\left[1-p_*^{\,j}(u\mid x)\right]c_j^-(u),
\label{eq:theory-retained-local-utility}
\end{equation}
where $c_j^+(u)$ and $c_j^-(u)$ are the task-level utilities associated with
successful and unsuccessful execution, respectively. Illegal actions are
excluded from the distributions below.

\begin{assumption}[Attribution and censoring]
\label{ass:theory-attribution}
Whenever $C=1$, the recorded outcome is correctly attributed to the selected
receiver action. Receiver-side probability and utility results are conditional
on $C=1$. They admit an unconditional interpretation only when
\[
Y\perp C\mid X,U,J.
\]
\end{assumption}

\begin{assumption}[Coverage]
\label{ass:theory-coverage}
Every legal context--candidate pair on which a receiver guarantee is invoked
is selected with positive probability under the retained rollout distribution.
\end{assumption}

\subsection{Proof of Proposition 1: Prospective Acquisition Target}
\label{app:theory-acquisition}

Let $\mathcal{M}_{\mathrm{acq}}$ denote the fixed acquisition model used to
construct SPL supervision. Probabilities and expectations in this subsection
are taken with respect to this model. Consider binary task-validity variables
$W_1,\ldots,W_n$. Given acquisition context $c_t^i$, let
\[
b_r
=
\Pr_{\mathcal{M}_{\mathrm{acq}}}(W_r=1\mid c_t^i),
\qquad
b=(b_1,\ldots,b_n).
\]
If a future collaborator can execute at most $k$ tasks, define
\begin{equation}
V_k(b)
=
\max_{S\subseteq\{1,\ldots,n\}:|S|\leq k}
\sum_{r\in S} b_r .
\label{eq:theory-top-k-value}
\end{equation}

For a legal acquisition action $u$, let $Z_u$ denote its observation and
$b^{Z_u}$ the posterior marginals produced by
$\mathcal{M}_{\mathrm{acq}}$. Its model-relative prospective gain is
\begin{equation}
\Delta_t^i(u)
=
\mathbb{E}_{\mathcal{M}_{\mathrm{acq}}}
\left[
V_k(b^{Z_u})
\mid c_t^i,u
\right]
-
V_k(b).
\label{eq:theory-acquisition-gain}
\end{equation}
SPL combines this gain with acquisition distance through
\begin{equation}
s_t^i(u)
=
\log\!\left(\max\{\Delta_t^i(u),\epsilon_{\mathrm{acq}}\}\right)
-
\frac{d_t^i(u)}{\tau_d}.
\label{eq:theory-acquisition-score}
\end{equation}
Let
\[
\mathcal{P}_t^i
=
\left\{
u:
M_t^i(u)=1,\;
u\text{ is task-directed},\;
\Delta_t^i(u)>0
\right\}
\]
denote the legal positive-gain task-directed actions.

For completeness, the task-directed component of the SPL acquisition target is
\begin{equation}
\widetilde q_{\mathrm{acq}}^i(u)
=
\frac{\exp(s_t^i(u))}
{\sum_{v\in\mathcal{P}_t^i}\exp(s_t^i(v))}.
\label{eq:theory-acquisition-gibbs}
\end{equation}
Equivalently,
\begin{equation}
\widetilde q_{\mathrm{acq}}^i
=
\arg\max_{q\in\Delta(\mathcal{P}_t^i)}
\left\{
\sum_{u\in\mathcal{P}_t^i}q(u)s_t^i(u)
+
H(q)
\right\},
\label{eq:theory-acquisition-variational}
\end{equation}
where $H(q)=-\sum_u q(u)\log q(u)$.

\begin{proof}
For every feasible set $S$ with $|S|\leq k$, let
\[
L_S(b)=\sum_{r\in S}b_r.
\]
Since
\[
V_k(b)=\max_{S:|S|\leq k}L_S(b),
\]
$V_k$ is convex. Bayesian consistency under
$\mathcal{M}_{\mathrm{acq}}$ gives
\[
\mathbb{E}_{\mathcal{M}_{\mathrm{acq}}}
\left[
b^{Z_u}\mid c_t^i,u
\right]
=
b.
\]
Jensen's inequality therefore gives
\[
\mathbb{E}_{\mathcal{M}_{\mathrm{acq}}}
\left[
V_k(b^{Z_u})
\mid c_t^i,u
\right]
\geq
V_k
\left(
\mathbb{E}_{\mathcal{M}_{\mathrm{acq}}}
[b^{Z_u}\mid c_t^i,u]
\right)
=
V_k(b),
\]
which proves
\[
\Delta_t^i(u)\geq0.
\]
This is the standard nonnegative value-of-information argument applied to the
top-$k$ completion surrogate.

For the Gibbs characterization, consider
\[
F(q)
=
\sum_{u\in\mathcal{P}_t^i}q(u)s_t^i(u)
+
H(q)
\]
over the probability simplex. Introducing a Lagrange multiplier for
$\sum_u q(u)=1$ gives the stationary condition
\[
s_t^i(u)-\log q(u)-1+\lambda=0,
\]
hence
\[
q(u)\propto\exp(s_t^i(u)).
\]
Normalization yields Eq.~\eqref{eq:theory-acquisition-gibbs}. Since entropy is
strictly concave and the score term is linear, this solution is the unique
maximizer of Eq.~\eqref{eq:theory-acquisition-variational}.
\end{proof}

The complete acquisition target assigns the prescribed fixed mass to the scan
action and distributes the remaining mass according to
$\widetilde q_{\mathrm{acq}}^i$. If $\mathcal{P}_t^i$ is empty, the scan
action receives unit target mass. The implementation uses the small
positive-gain threshold and numerical constants reported in
Sec.~\ref{app:task-model-utilities}.

The nonnegativity result is model-relative. Bayesian consistency within
$\mathcal{M}_{\mathrm{acq}}$ does not require its assumed priors to equal the
true environment priors. Prior misspecification can therefore preserve
nonnegative model-relative gains while changing how those gains relate to
performance in the true environment; this sensitivity is evaluated in
Sec.~\ref{app:prior-misspecification}.

\subsection{Proof of Proposition 2: Retained-Outcome Calibration}
\label{app:theory-outcome}

For any measurable predictor $p^j(u\mid x)\in[0,1]$, define its population
binary-cross-entropy risk over retained outcomes as
\begin{equation}
\mathcal{L}_{\mathrm{out}}(p)
=
\mathbb{E}
\left[
\ell_{\mathrm{BCE}}
\left(
Y,p^J(U\mid X)
\right)
\mid C=1
\right].
\label{eq:theory-population-bce}
\end{equation}

The excess population risk satisfies
\begin{equation}
\mathcal{L}_{\mathrm{out}}(p)
-
\mathcal{L}_{\mathrm{out}}(p_*)
=
\mathbb{E}
\left[
\operatorname{KL}
\left(
\operatorname{Ber}(p_*^J(U\mid X))
\Vert
\operatorname{Ber}(p^J(U\mid X))
\right)
\mid C=1
\right].
\label{eq:theory-bce-excess-kl}
\end{equation}

Let
\[
C_\Delta
=
\sup_{u,j}
|c_j^+(u)-c_j^-(u)|
<
\infty.
\]
The induced local utility then satisfies
\begin{equation}
\mathbb{E}
\left[
\left(
U_p^J(U\mid X)
-
U_*^J(U\mid X)
\right)^2
\mid C=1
\right]
\leq
\frac{C_\Delta^2}{2}
\left[
\mathcal{L}_{\mathrm{out}}(p)
-
\mathcal{L}_{\mathrm{out}}(p_*)
\right].
\label{eq:theory-local-utility-error}
\end{equation}

\begin{proof}
Fix a covered triple $(x,u,j)$ and abbreviate
\[
p_*=p_*^{\,j}(u\mid x),
\qquad
p=p^j(u\mid x).
\]
The conditional binary-cross-entropy risk is
\[
r(p;p_*)
=
-p_*\log p
-
(1-p_*)\log(1-p).
\]
Its excess risk is
\[
r(p;p_*)-r(p_*;p_*)
=
\operatorname{KL}
\left(
\operatorname{Ber}(p_*)
\Vert
\operatorname{Ber}(p)
\right),
\]
which is non-negative and vanishes only when $p=p_*$. Taking expectations
over $(X,U,J)\mid C=1$ gives
Eq.~\eqref{eq:theory-bce-excess-kl}.

From Eq.~\eqref{eq:theory-retained-local-utility},
\[
U_p^j(u\mid x)-U_*^{\,j}(u\mid x)
=
\left[p^j(u\mid x)-p_*^{\,j}(u\mid x)\right]
\left[c_j^+(u)-c_j^-(u)\right].
\]
Hence
\[
\left|
U_p^j-U_*^{\,j}
\right|^2
\leq
C_\Delta^2|p-p_*|^2.
\]
Pinsker's inequality gives
\[
|p-p_*|^2
\leq
\frac{1}{2}
\operatorname{KL}
\left(
\operatorname{Ber}(p_*)
\Vert
\operatorname{Ber}(p)
\right).
\]
Combining the two inequalities and taking expectations proves
Eq.~\eqref{eq:theory-local-utility-error}.
\end{proof}

The result is conditional on retained attributable outcomes. Under the
additional condition
\[
Y\perp C\mid X,U,J,
\]
the retained-outcome probability also equals the corresponding unconditional
success probability.

\subsection{Receiver Target and Supporting Analysis}
\label{app:theory-receiver-target}

Fix an eligible receiver decision $(j,t)$ with legal action set
\[
\mathcal{A}=\mathcal{A}_t^j,
\qquad
K=|\mathcal{A}|.
\]
For any local utility vector $U$ over $\mathcal{A}$, let $\mu_U$ and
$\sigma_U^2$ denote its mean and variance over legal actions, and define the
standardized score
\begin{equation}
z_U(u)
=
\frac{
U(u)-\mu_U
}{
\sqrt{\sigma_U^2+\epsilon_z}
},
\qquad
\epsilon_z>0.
\label{eq:theory-standardised-score}
\end{equation}

For a score vector $z$, the SPL receiver target is
\begin{equation}
q_{\mathrm{SPL}}^j(u)
=
\frac{
\exp(\beta z(u))
}{
\sum_{v\in\mathcal{A}}\exp(\beta z(v))
},
\qquad
u\in\mathcal{A}.
\label{eq:theory-spl-target}
\end{equation}

\begin{lemma}[Soft receiver target]
\label{lem:theory-receiver-target}
For $\beta>0$, the distribution in Eq.~\eqref{eq:theory-spl-target}
uniquely maximizes
\begin{equation}
\sum_{u\in\mathcal{A}}q(u)z(u)
+
\frac{1}{\beta}H(q)
\label{eq:theory-receiver-variational}
\end{equation}
over $q\in\Delta(\mathcal{A})$, and satisfies
\begin{equation}
\max_{u\in\mathcal{A}}z(u)
-
\mathbb{E}_{u\sim q_{\mathrm{SPL}}^j}[z(u)]
\leq
\frac{\log K}{\beta}.
\label{eq:theory-temperature-gap}
\end{equation}
\end{lemma}

\begin{proof}
Introducing a Lagrange multiplier for the simplex constraint gives the
stationary condition
\[
z(u)
-
\frac{1}{\beta}
\left[\log q(u)+1\right]
+
\lambda
=
0,
\]
hence
\[
q(u)\propto\exp(\beta z(u)).
\]
Since the entropy term is strictly concave, this solution is unique and equals
Eq.~\eqref{eq:theory-spl-target}.

Let
\[
u^*\in\arg\max_{u\in\mathcal{A}}z(u).
\]
Optimality of $q_{\mathrm{SPL}}^j$ relative to the point mass on $u^*$ gives
\[
\mathbb{E}_{q_{\mathrm{SPL}}^j}[z]
+
\frac{1}{\beta}H(q_{\mathrm{SPL}}^j)
\geq
z(u^*).
\]
Using
\[
H(q_{\mathrm{SPL}}^j)\leq\log K
\]
yields Eq.~\eqref{eq:theory-temperature-gap}.
\end{proof}

\paragraph{Receiver distillation.}
For a fixed SPL target $q_{\mathrm{SPL}}^j$ and receiver policy $\pi_t^j$,
\begin{equation}
\operatorname{CE}(q_{\mathrm{SPL}}^j,\pi_t^j)
=
H(q_{\mathrm{SPL}}^j)
+
\operatorname{KL}(q_{\mathrm{SPL}}^j\Vert\pi_t^j).
\label{eq:theory-receiver-kl}
\end{equation}
Thus,
\begin{equation}
\eta
=
\operatorname{KL}(q_{\mathrm{SPL}}^j\Vert\pi_t^j)
=
\operatorname{CE}(q_{\mathrm{SPL}}^j,\pi_t^j)
-
H(q_{\mathrm{SPL}}^j),
\label{eq:theory-measurable-eta}
\end{equation}
so the local distillation term is directly computable from the receiver target
and actor distributions. Averaging
Eq.~\eqref{eq:theory-measurable-eta} over receiver decisions gives an
empirical diagnostic of target distillation.

\paragraph{Connection between outcome error and score error.}
The calibration result in Eq.~\eqref{eq:theory-local-utility-error} controls
utility error on selected actions, whereas receiver regret depends on scores
over the full legal action set. A quantitative connection follows under
stronger coverage. Let
\[
\rho^j(u\mid x)
=
\Pr(U=u\mid X=x,J=j,C=1).
\]
Suppose that, on the retained context distribution considered below, there
exists $\alpha>0$ such that
\[
\rho^j(u\mid x)\geq\alpha
\]
for every legal candidate whose utility is supplied by the outcome model.
Any remaining legal actions with fixed utilities are assumed identical in the
predicted and reference utility vectors.

\begin{lemma}[Mean score-error control]
\label{lem:theory-score-error-control}
Let $p_{\bar{\psi}}$ denote the target outcome predictor and let
\[
\delta(X,J)
=
\|z_{\bar{\psi}}-z_*\|_\infty.
\]
Under the uniform coverage condition above,
\begin{equation}
\mathbb{E}
\left[
\delta(X,J)^2
\mid C=1
\right]
\leq
\frac{C_\Delta^2}
{2\alpha\epsilon_z}
\left[
\mathcal{L}_{\mathrm{out}}(p_{\bar{\psi}})
-
\mathcal{L}_{\mathrm{out}}(p_*)
\right]
\leq
\frac{C_\Delta^2}
{2\alpha\epsilon_z}
\mathcal{L}_{\mathrm{out}}(p_{\bar{\psi}}).
\label{eq:theory-score-error-control}
\end{equation}
\end{lemma}

\begin{proof}
For a fixed legal action set, the stabilized standardization map in
Eq.~\eqref{eq:theory-standardised-score} is
$1/\sqrt{\epsilon_z}$-Lipschitz from the utility-vector $\ell_2$ norm to the
score-vector $\ell_\infty$ norm. Hence
\[
\delta(X,J)^2
\leq
\frac{
\|U_{\bar{\psi}}-U_*\|_2^2
}{
\epsilon_z
}.
\]
The uniform coverage condition gives
\[
\sum_{u\in\mathcal{A}}
\left(
U_{\bar{\psi}}^j(u\mid x)
-
U_*^j(u\mid x)
\right)^2
\leq
\frac{1}{\alpha}
\mathbb{E}
\left[
\left(
U_{\bar{\psi}}^J(U\mid X)
-
U_*^J(U\mid X)
\right)^2
\mid X=x,J=j,C=1
\right].
\]
Taking expectations over retained contexts and applying
Eq.~\eqref{eq:theory-local-utility-error} proves the first inequality in
Eq.~\eqref{eq:theory-score-error-control}. The second follows from
\[
\mathcal{L}_{\mathrm{out}}(p_*)\geq0.
\]
\end{proof}

Equation~\eqref{eq:theory-score-error-control} controls mean squared score
error under the retained context distribution; it is not a uniform
per-context bound. For fixed $\alpha$ and $\epsilon_z$, vanishing population
excess outcome risk implies vanishing mean squared score error. Held-out
binary-cross-entropy provides an empirical estimate of population outcome risk,
while finite-sample generalization remains separate from this population
statement.

\subsection{Proof of Theorem 1: Local Receiver Score Regret}
\label{app:theory-receiver-regret}

Let $z_*$ denote the standardized score induced by the retained-outcome local
utility $U_*^{\,j}$, and let $z_{\bar{\psi}}$ denote the score induced by the
target outcome model. Define
\[
\delta
=
\|z_{\bar{\psi}}-z_*\|_\infty,
\qquad
\eta
=
\operatorname{KL}
\left(
q_{\mathrm{SPL}}^j\Vert\pi_t^j
\right),
\]
where
\[
q_{\mathrm{SPL}}^j(u)
\propto
\exp(\beta z_{\bar{\psi}}(u)),
\]
and let
\[
R_z
=
\max_{u\in\mathcal{A}}z_*(u)
-
\min_{u\in\mathcal{A}}z_*(u).
\]

Theorem~1 in the main paper states that
\begin{equation}
\operatorname{Reg}_{z,t}^j(\pi_t^j)
\leq
\frac{\log K}{\beta}
+
2\delta
+
R_z\sqrt{\frac{\eta}{2}},
\label{eq:theory-receiver-regret}
\end{equation}
where
\begin{equation}
\operatorname{Reg}_{z,t}^j(\pi_t^j)
=
\max_{u\in\mathcal{A}}z_*(u)
-
\mathbb{E}_{u\sim\pi_t^j}[z_*(u)].
\label{eq:theory-score-regret-definition}
\end{equation}

\begin{proof}
Add and subtract the expected reference score under
$q_{\mathrm{SPL}}^j$:
\[
\operatorname{Reg}_{z,t}^j(\pi_t^j)
=
\left[
\max_u z_*(u)
-
\mathbb{E}_{q_{\mathrm{SPL}}^j}[z_*]
\right]
+
\left[
\mathbb{E}_{q_{\mathrm{SPL}}^j}[z_*]
-
\mathbb{E}_{\pi_t^j}[z_*]
\right].
\]

Since
\[
\|z_{\bar{\psi}}-z_*\|_\infty\leq\delta,
\]
we have
\[
\max_u z_*(u)
-
\mathbb{E}_{q_{\mathrm{SPL}}^j}[z_*]
\leq
\max_u z_{\bar{\psi}}(u)
-
\mathbb{E}_{q_{\mathrm{SPL}}^j}[z_{\bar{\psi}}]
+
2\delta.
\]
Applying Eq.~\eqref{eq:theory-temperature-gap} gives
\[
\max_u z_*(u)
-
\mathbb{E}_{q_{\mathrm{SPL}}^j}[z_*]
\leq
\frac{\log K}{\beta}
+
2\delta.
\]

For the second term,
\[
\mathbb{E}_{q_{\mathrm{SPL}}^j}[z_*]
-
\mathbb{E}_{\pi_t^j}[z_*]
\leq
R_z\,
\operatorname{TV}
\left(
q_{\mathrm{SPL}}^j,
\pi_t^j
\right).
\]
Pinsker's inequality gives
\[
\operatorname{TV}
\left(
q_{\mathrm{SPL}}^j,
\pi_t^j
\right)
\leq
\sqrt{
\frac{
\operatorname{KL}(q_{\mathrm{SPL}}^j\Vert\pi_t^j)
}{2}
}
=
\sqrt{\frac{\eta}{2}}.
\]
Combining the two bounds proves
Eq.~\eqref{eq:theory-receiver-regret}.
\end{proof}

The bound separates three local sources of receiver error:
finite-temperature softness, error in the evidence-conditioned score estimate,
and incomplete distillation of the SPL target into the decentralized policy.
Equation~\eqref{eq:theory-measurable-eta} makes the distillation term directly
observable during training, while
Lemma~\ref{lem:theory-score-error-control} gives a population-level connection
between score approximation and outcome-model error under uniform coverage.

\subsection{Scope of the Local Analysis}
\label{app:theory-guarantee-boundary}

The results above characterize the local auxiliary objectives used by SPL.
They support the construction and interpretation of the training signals but
do not establish global convergence or monotonic improvement of the
cooperative return.

\paragraph{Acquisition model and prior misspecification.}
Proposition~1 in the main paper is a model-relative value-of-information
result. Bayesian consistency within the acquisition model guarantees
non-negative prospective gain for the top-$k$ surrogate, including when the
model uses internally consistent but misspecified priors. Such
misspecification can change the relation between the acquisition score and
performance under the true environment distribution.
Section~\ref{app:prior-misspecification} evaluates this sensitivity while
keeping the remaining structured model components fixed.

\paragraph{Receiver approximation scope.}
Proposition~2 in the main paper controls expected local-utility error on the
retained selected-action distribution. The receiver distillation quantity
$\eta$ is directly computable from the target and actor distributions through
Eq.~\eqref{eq:theory-measurable-eta}.
Lemma~\ref{lem:theory-score-error-control} further connects mean squared
standardized-score error to population outcome-model risk under a uniform
retained-action coverage condition. It does not provide a uniform per-context
bound, and empirical validation loss introduces the usual finite-sample
generalization error. Theorem~1 in the main paper therefore remains a local
standardized-score guarantee, not a bound on raw utility or full team return.

\paragraph{Local utility and retained outcomes.}
The retained-outcome utility in
Eq.~\eqref{eq:theory-retained-local-utility} describes the task-level
consequence of an attributable receiver action and need not equal the full
Dec-POMDP action value. Receiver-side results are conditional on attributable
outcomes retained in $\mathcal{D}_{\mathrm{out}}$ and admit an unconditional
interpretation only under Assumption~\ref{ass:theory-attribution}. Under
partially overlapping participation, the receiver context may combine evidence
acquired before and after participation begins; the analysis conditions on
the complete actor-visible context and does not isolate the causal contribution
of pre-participation evidence.

\paragraph{Backbone optimization.}
SPL auxiliary actor updates are applied in addition to the native MAPPO, IPPO,
or HAPPO optimization. The analysis does not provide a convergence guarantee
for the resulting joint optimization; the reported task-completion and return
results are empirical.

\section{Implementation and Reproducibility Details}
\label{app:implementation-details}

This section provides the environment configurations, participation schedules, SPL target construction, optimization settings, and training procedure used in the reported experiments. All methods are implemented within the HARL framework. We denote the SPL extensions of MAPPO, IPPO, and HAPPO as SPL-MAPPO, SPL-IPPO, and SPL-HAPPO, respectively. Within each matched comparison, the corresponding unmodified backbone is referred to as MAPPO, IPPO, or HAPPO.

\subsection{Environment Configuration}
\label{app:environment-construction}

The primary four-agent MPE and RWARE evaluations contain two earlier participants, two future collaborators, and eight candidate tasks. Each domain includes Base and Advanced variants. We additionally evaluate SPL in physics-based UAV--UGV tasks in Isaac Lab with either one UAV and one UGV or two UAVs and two UGVs. Persistent task-relevant evidence produced by earlier participants remains available to future collaborators.

The Base variants use fixed task assignments and earlier-participant--future-collaborator correspondences. MPE-Advanced randomizes balanced task shards, spawn assignments, and future-collaborator identities, while RWARE-Advanced permutes task identifiers and uses a shared candidate pool with dynamic task claims. The Isaac Lab settings introduce physical motion, obstacles, and planner-mediated downstream task execution. The eight-agent scaling configurations are described separately below.

\begin{table}[!htbp]
\centering
\caption{Four-agent MPE/RWARE and Isaac Lab configurations.}
\label{tab:app-environment-summary}
\small
\setlength{\tabcolsep}{5pt}
\renewcommand{\arraystretch}{1.06}
\begin{adjustbox}{max width=\linewidth}
\begin{tabular}{lcccc}
\toprule
Environment & Earlier & Future & Tasks & Horizon \\
\midrule
MPE-Base       & 2 & 2 & 8 & 125 (250) \\
MPE-Advanced   & 2 & 2 & 8 & 125 (250) \\
RWARE-Base     & 2 & 2 & 8 & 500 (1000) \\
RWARE-Advanced & 2 & 2 & 8 & 500 (1000) \\
Isaac Lab (2-agent) & 1 UAV & 1 UGV & 4 & 600 \\
Isaac Lab (4-agent) & 2 UAVs & 2 UGVs & 6 & 840 \\
\bottomrule
\end{tabular}
\end{adjustbox}
\end{table}

For the primary four-agent MPE and RWARE settings, the environments use the same eight task-validity priors and a fixed-valid-task sampling rule. Probe outcomes are noisy but persistent, and each future collaborator can execute at most two candidate tasks. Table~\ref{tab:app-task-settings} summarizes the main task and evidence parameters.

\begin{table}[!htbp]
\centering
\caption{Task and evidence settings for MPE and RWARE.}
\label{tab:app-task-settings}
\small
\setlength{\tabcolsep}{5pt}
\renewcommand{\arraystretch}{1.06}
\begin{adjustbox}{max width=\linewidth}
\begin{tabular}{lcc}
\toprule
Setting & MPE & RWARE \\
\midrule
Candidate tasks & 8 & 8 \\
Validity priors
& \multicolumn{2}{c}{$(0.80,0.80,0.65,0.65,0.35,0.35,0.20,0.20)$} \\
Probe budget / earlier participant & 5 & 6 \\
Probe limit / task & 2 & 2 \\
Probe accuracy & 0.85 & 0.85 \\
Publication delay & 5 & 5 \\
Execution budget / future collaborator & 2 & 2 \\
Valid-task value & $\{12,12,10,10,8,8,6,6\}$ & 1 \\
Probe cost & 0.2 & 0.01 \\
Invalid-execution cost & 8.0 & 0.25 \\
\bottomrule
\end{tabular}
\end{adjustbox}
\end{table}

\subsection{Temporal Participation Patterns}
\label{app:temporal-patterns}

Participation intervals follow the half-open convention $[t_{\mathrm{start}},t_{\mathrm{end}})$. Inactive agents take a null environment action and make no effective policy decision. The Base and Advanced variants within each domain use the same participation schedules.

We evaluate five participation patterns. \emph{Non-overlap} uses a strict team-wide handoff from earlier participants to future collaborators. \emph{Overlap-25} and \emph{Overlap-50} introduce fixed levels of cross-role overlap. \emph{Dynamic} samples pair-specific handoff times each episode while preserving aggregate role exposure. \emph{Cascaded} distributes the two earlier-participant--future-collaborator pairs across successive participation intervals, with an extended horizon to preserve total role exposure.

\begin{table}[!htbp]
\centering
\caption{Four-agent participation schedules.}
\label{tab:app-temporal-windows}
\small
\setlength{\tabcolsep}{3.2pt}
\renewcommand{\arraystretch}{1.04}

\begin{adjustbox}{max width=\linewidth}
\begin{tabular}{@{}llccccc@{}}
\toprule
Domain & Participation & Horizon & $P_0$ & $P_1$ & $R_0$ & $R_1$ \\
\midrule

MPE & Non-overlap & 125
& $[0,45)$ & $[0,45)$ & $[45,125)$ & $[45,125)$ \\

MPE & Overlap-25 & 125
& $[0,45)$ & $[0,45)$ & $[34,125)$ & $[34,125)$ \\

MPE & Overlap-50 & 125
& $[0,45)$ & $[0,45)$ & $[22,125)$ & $[22,125)$ \\

MPE & Dynamic & 125
& $[0,h_0)$ & $[0,h_1)$ & $[h_0,125)$ & $[h_1,125)$ \\

MPE & Cascaded & 250
& $[0,45)$ & $[125,170)$ & $[45,125)$ & $[170,250)$ \\

\midrule

RWARE & Non-overlap & 500
& $[0,200)$ & $[0,200)$ & $[200,500)$ & $[200,500)$ \\

RWARE & Overlap-25 & 500
& $[0,175)$ & $[0,225)$ & $[175,500)$ & $[225,500)$ \\

RWARE & Overlap-50 & 500
& $[0,150)$ & $[0,250)$ & $[150,500)$ & $[250,500)$ \\

RWARE & Dynamic & 500
& $[0,h_0)$ & $[0,h_1)$ & $[h_0,500)$ & $[h_1,500)$ \\

RWARE & Cascaded & 1000
& $[0,200)$ & $[500,700)$ & $[200,500)$ & $[700,1000)$ \\

\bottomrule
\end{tabular}
\end{adjustbox}
\end{table}

Here, $P_0,P_1$ denote earlier participants and $R_0,R_1$ denote future collaborators. In MPE, Overlap-25 and Overlap-50 provide 11 and 23 steps of cross-role overlap, respectively; the corresponding RWARE overlaps are 50 and 100 steps. For Dynamic participation,
$h_0=\bar h+\delta$ and $h_1=\bar h-\delta$, where $\delta$ is sampled once per episode from a symmetric discrete uniform distribution. MPE uses $\bar h=45$ and $\delta\in\{-11,\ldots,11\}$, while RWARE uses $\bar h=200$ and $\delta\in\{-50,\ldots,50\}$. Consequently, Dynamic changes the pair-specific handoff times while preserving aggregate earlier-participant and future-collaborator exposure. Cascaded preserves the same role exposure while changing its temporal ordering.

\subsection{Isaac Lab Configuration}
\label{app:isaac-configuration}

The two-agent Isaac Lab task contains four candidate sites, of which two are serviceable in each episode, on a $12\times12$ occupancy grid with five barriers. The four-agent setting contains six candidate sites, of which four are serviceable. Candidate positions are public from reset, while serviceability remains hidden. UAV sensing has accuracy $0.85$ and a five-step publication delay. UGV motion is planned from the published occupancy map, and selecting a task commits the UGV to its physical execution. Neither the actor nor the critic receives hidden serviceability.

\begin{table}[!htbp]
\centering
\caption{Isaac Lab participation and training configurations.}
\label{tab:app-isaac-settings}
\small
\setlength{\tabcolsep}{4pt}
\renewcommand{\arraystretch}{1.06}
\begin{adjustbox}{max width=\linewidth}
\begin{tabular}{@{}lll>{\raggedright\arraybackslash}p{3.7cm}@{}}
\toprule
Agents & Participation & Training steps & Participation windows \\
\midrule
1 UAV + 1 UGV
& Non-overlap
& 1,996,800
& UAV $[0,200)$; UGV $[200,600)$ \\

1 UAV + 1 UGV
& Dynamic
& 1,996,800
& UAV $[0,h)$; UGV $[h,600)$, $h\sim\operatorname{Unif}\{150,\ldots,250\}$ \\

2 UAVs + 2 UGVs
& Non-overlap
& 3,010,560
& UAVs $[0,300)$; UGVs $[300,840)$ \\

2 UAVs + 2 UGVs
& Dynamic
& 3,010,560
& UAVs $[0,300)$; UGVs $[h,h+540)$, $h\in\{150,180,\ldots,300\}$ \\
\bottomrule
\end{tabular}
\end{adjustbox}
\end{table}

Dynamic handoff times are sampled per episode. The two-agent setting exposes the normalized handoff step to the actors, while the four-agent setting uses a shared UGV participation window with a sampled start time. MAPPO and SPL-MAPPO use matched participation schedules and evaluation instances. Each setting uses three independent training seeds and 200 final evaluation episodes per seed. Final policies are evaluated deterministically, and checkpoints are not selected by performance.

\subsection{SPL Target Construction}
\label{app:task-model-utilities}

\paragraph{Prospective acquisition target.}
The frozen acquisition model constructs the task-validity posterior from the configured validity priors, the fixed number of valid tasks, and public probe outcomes available in the actor-visible context. Repeated probe observations are aggregated through their positive and negative counts using the configured probe accuracy. Hidden task validity and future outcomes are never provided to either the actor or the target-construction procedure.

Posterior marginals $b_r=\Pr(W_r=1\mid c_t^i)$ are used to compute the prospective gain $\Delta_t^i(u)$ defined in Eq.~\eqref{eq:theory-acquisition-gain}. In the implementation, task-directed actions with
$\Delta_t^i(u)\leq10^{-8}$ receive no acquisition-target mass. For the remaining legal task-directed actions,
\begin{equation}
w_t^i(u)
=
\max\{\Delta_t^i(u),\epsilon_{\mathrm{acq}}\}
\exp\!\left(
-\frac{d_t^i(u)}{\tau_d}
\right),
\qquad
\epsilon_{\mathrm{acq}}=10^{-12},
\quad
\tau_d=0.50.
\label{eq:app-acquisition-weight}
\end{equation}
The weights are normalized over eligible task-directed actions. The target assigns $0.95$ total probability mass to this distribution and $0.05$ to the scan action. If no positive-gain task-directed action is legal, the scan action receives unit target mass. The acquisition model and target-construction rule remain fixed throughout training.

\paragraph{Outcome-supervised receiver target.}
For each future collaborator $j$, SPL uses a candidate-shared outcome model
\begin{equation}
p_{\psi}^{\,j}(u\mid x)
=
\operatorname{sigmoid}
\left(
f_{\psi}(g_j(x,u))
\right),
\label{eq:app-outcome-model}
\end{equation}
where $g_j(x,u)$ contains the actor-visible pre-decision context, candidate-action features, and collaborator-role identity. The same scorer $f_\psi$ is applied across candidate actions. Training labels are available only for selected actions whose outcomes can be reliably attributed; unselected actions receive no counterfactual labels.

A slowly updated target model $p_{\bar{\psi}}$ is used to construct the receiver target. For legal candidate action $u$,
\begin{equation}
U_{\bar{\psi}}^{\,j}(u\mid x)
=
p_{\bar{\psi}}^{\,j}(u\mid x)c_j^+(u)
+
\left[
1-p_{\bar{\psi}}^{\,j}(u\mid x)
\right]c_j^-(u).
\label{eq:app-local-utility}
\end{equation}
The task-level utilities are fixed during training. In MPE,
$c_j^+(u)=v(u)/12$ with $v(u)\in\{12,10,8,6\}$ and
$c_j^-(u)=-8/12$.
In RWARE,
$c_j^+(u)=1$ and $c_j^-(u)=-0.25$.
In Isaac Lab,
$c_j^+(u)=1$ and $c_j^-(u)=0$.
The no-op utility is zero. These quantities define the local utilities used for SPL supervision and are not interpreted as the complete environment return.

Utilities are standardized over the current legal action set using
$\epsilon_z=10^{-6}$ as the variance stabilizer, following
Eq.~\eqref{eq:theory-standardised-score}. The receiver target is then the Boltzmann distribution in Eq.~\eqref{eq:theory-spl-target}. All reported experiments use $\beta=1$. Receiver supervision uses this soft utility target directly, without mixing it with the current actor policy or converting it to a one-hot target.

\begin{table}[!htbp]
\centering
\caption{Training hyperparameters for the PPO backbones and SPL.}
\label{tab:app-hyperparameters}
\small
\setlength{\tabcolsep}{6pt}
\renewcommand{\arraystretch}{1.10}

\begin{adjustbox}{max width=\linewidth}
\begin{tabular}{@{}lc@{\hspace{1.2cm}}lc@{}}
\toprule
\multicolumn{2}{c}{\textbf{PPO backbone}}
&
\multicolumn{2}{c}{\textbf{SPL-specific}} \\
\cmidrule(lr){1-2}
\cmidrule(lr){3-4}
Parameter & Value & Parameter & Value \\
\midrule

Actor hidden layers
& $[128,128]$
& Outcome hidden width
& 128 \\

Activation
& ReLU
& Outcome learning rate
& $3\times10^{-4}$ \\

Actor learning rate
& $1\times10^{-4}$
& Replay capacity
& 20,000 \\

Critic learning rate
& $5\times10^{-4}$
& Outcome minibatch size
& 128 \\

PPO epochs
& 3
& Outcome gradient steps
& 4 \\

Critic epochs
& 3
& Target EMA rate $\rho$
& 0.05 \\

PPO clip
& 0.2
& Outcome-model gate
& 128 \\

Entropy coefficient
& 0.002
& Per-class model gate
& 16 \\

Discount $\gamma$
& 0.99
& Receiver-policy gate
& 1,024 \\

GAE $\lambda$
& 0.95
& Per-class receiver gate
& 128 \\

Value normalization
& enabled
& Acquisition epochs
& 1 \\

Recurrent policy
& disabled
& Receiver epochs
& 20 \\

Parameter sharing
& disabled
& Acquisition task mass
& 0.95 \\

Gradient norm cap
& 10
& Distance temperature $\tau_d$
& 0.50 \\

&
&
Receiver target strength $\beta$
& 1.0 \\

\bottomrule
\end{tabular}
\end{adjustbox}
\end{table}

\subsection{Optimization and Training Settings}
\label{app:hyperparameters}

MAPPO, IPPO, and HAPPO use the same feed-forward actor architecture and PPO hyperparameters listed in Table~\ref{tab:app-hyperparameters}. Policies are not parameter-shared within a team. MAPPO and HAPPO retain centralized critics, whereas IPPO uses independent critics based on local observations. SPL does not use centralized critic state and does not modify the native reward, return estimator, critic objective, or critic target. The corresponding augmented methods are denoted SPL-MAPPO, SPL-IPPO, and SPL-HAPPO. Vanilla--SPL comparisons use the same backbone configuration, environment, participation schedule, training budget, and evaluation protocol.

All primary four-agent MPE and RWARE method--environment--participation combinations are trained for two million environment steps without early stopping. Reported results use five independent training seeds per method and configuration, with 200 final evaluation episodes per seed. Training seeds are the independent replication units. Final policies are evaluated deterministically using a fixed 200-episode reset-seed schedule reused across training seeds within each environment protocol. The scaling and Isaac Lab settings use the training procedures described in their respective configuration sections.

\paragraph{Training-seed pairing.}
A retrospective audit recovered all 600 expected method--training-seed records for the 60 primary Vanilla--SPL comparisons. Same-training-seed pairing was verified for 285 pairs across 57 configurations, providing five paired training seeds in each of those configurations. The Non-overlap MAPPO comparisons in MPE-Advanced, RWARE-Base, and RWARE-Advanced use different Vanilla and SPL training-seed sets. Their observed method means remain included in the descriptive results, but paired confidence intervals and paired $p$-values are reported as unavailable. Consequently, paired inference covers 57 of the 60 primary configurations. Full statistical details are provided in Appendix~D.7.

\paragraph{Baseline implementation and adaptation.}
MAE-MAPPO uses the same per-agent observations, legal task-option masks, team rewards, participation schedules, and evaluation protocol as MAPPO. For Sable and SMPE$^2$, we retain the original learning architectures and optimization objectives and adapt only the environment interface to the SP tasks. Both methods receive the same environment-specific local observations, legal task-option masks, shared rewards, participation schedules, and episode boundaries as the PPO-based methods. A common task-option controller maps each legal high-level choice to the corresponding native environment action, while inactive agents receive only the environment-defined no-op action. Where global or state-level inputs are required, they are constructed solely from non-privileged observable information and do not expose hidden task validity. Sable retains its sequence-model rollout procedure, while SMPE$^2$ retains its recurrent policy, critic, state-VAE, and state-filtering components. All primary four-agent baseline comparisons use the same two-million-step training budget, five independent training seeds per method and configuration, and 200 final evaluation episodes per seed.

\paragraph{Knowledge-matched control.}
SKD-MAPPO uses the same structured task information, prospective acquisition supervision, task utilities, outcome-model training procedure, target-model update, readiness conditions, and auxiliary update budget as SPL-MAPPO. Its online outcome model is trained from the same complete actor-visible pre-decision context and attributable outcome labels as in SPL. When the learned target outcome model is subsequently used to construct receiver supervision, however, predecessor-generated task evidence is masked from its input. The decentralized actor observation remains unchanged. SKD-MAPPO therefore retains evidence during outcome-model learning but withholds that evidence when constructing the receiver target. This control tests the effect of removing inherited evidence specifically from the target-construction pathway while leaving the remaining SPL training machinery unchanged.

\paragraph{Evidence-redacted control.}
Evidence-redacted uses the same MAPPO backbone, acquisition supervision, task utilities, attributable outcome labels, readiness conditions, receiver utility construction, and auxiliary update budget as SPL-MAPPO. Unlike SKD-MAPPO, inherited predecessor-generated task evidence is removed consistently from the feature vectors supplied to both the online outcome model during outcome-model training and the target outcome model during receiver-target construction. The replay buffer retains the original actor-visible pre-decision context, and redaction is applied when outcome-model features are constructed. The decentralized actor observation is not modified. Evidence-redacted therefore evaluates receiver learning when inherited task evidence is unavailable to both outcome prediction and receiver-target construction, while preserving the remaining training and execution pipeline.

The distinction between the two controls is thus the stage at which inherited evidence is removed. SKD-MAPPO allows the outcome model to learn from the complete evidence-conditioned context but removes predecessor evidence when that model is queried to construct receiver supervision. Evidence-redacted removes the same class of inherited evidence from both outcome-model learning and receiver-target construction. Neither control removes the evidence from the decentralized actor observation, and neither exposes hidden task validity or future outcomes to the actor or critic.

\subsection{Training Procedure and Decentralized Execution}
\label{app:spl-training}

\paragraph{Outcome replay and readiness.}
When a future collaborator selects a task-directed action, SPL stores its actor-visible pre-decision context, selected action, collaborator-role identity, and subsequently attributable outcome. Successful completion is labeled $1$, while an attributable unsuccessful attempt is labeled $0$. Interrupted, expired, unresolved, or otherwise ambiguous attempts are excluded and receive no synthetic labels. Information observed after the decision is used only to determine the outcome label and is not included in the stored pre-decision context.

Outcome-model optimization begins once $\mathcal{D}_{\mathrm{out}}$ contains at least 128 retained decisions, including at least 16 examples from each outcome class. Each eligible training iteration then applies four outcome-model gradient steps with minibatches of at most 128 replay samples. Receiver-policy supervision begins once the replay contains at least 1,024 retained decisions with at least 128 examples from each class. Before this condition is met, future collaborators receive only their native MARL actor update. Prospective acquisition supervision does not depend on the receiver-side readiness conditions.

\paragraph{Update order and gradient boundaries.}
After each rollout, attributable receiver outcomes are appended to $\mathcal{D}_{\mathrm{out}}$. Once the outcome-model readiness condition is satisfied, the online outcome model is updated by binary cross-entropy, followed by the exponential-moving-average update
\[
\bar{\psi}
\leftarrow
(1-\rho)\bar{\psi}
+
\rho\psi.
\]
When receiver supervision is ready, receiver targets are constructed from the updated target outcome model before backbone optimization.

SPL actor updates are interleaved with the native actor updates. For each eligible actor, the native backbone actor update is performed first, followed by the corresponding SPL update. MAPPO processes actors in fixed agent order and updates the centralized critic after all actor updates. IPPO applies the native actor update, SPL update, and independent critic update for each agent before proceeding to the next agent.

HAPPO retains its sequential heterogeneous actor-update order. For each actor, the current cumulative factor is stored before the native actor update, followed by the applicable SPL update. Action probabilities are then re-evaluated to update the cumulative factor used by subsequent actors. Thus, later actors can receive factors reflecting the post-SPL policies of earlier actors. We do not claim that SPL-HAPPO preserves HAPPO's original monotonic-improvement guarantee.

The outcome loss updates only the online outcome model $\psi$, the acquisition loss updates only the eligible earlier-participant actor, and the receiver loss updates only the eligible future-collaborator actor. Gradients do not propagate through the frozen acquisition model, target outcome model, utility construction, replayed outcomes, future trajectories, or environment dynamics. SPL does not modify the environment reward, transition function, observation space, action space, native return estimator, or critic objective.

\paragraph{Decentralized execution.}
At execution time, each active agent acts solely through its decentralized policy,
\begin{equation}
u_t^i
\sim
\pi_{\theta_i}
\left(
\cdot\mid\tau_t^i,M_t^i
\right).
\label{eq:app-decentralised-execution}
\end{equation}
The acquisition model, outcome replay, online and target outcome models, utility calculations, and SPL target distributions are used only during training. SPL-MAPPO, SPL-IPPO, and SPL-HAPPO use the original decentralized policy architecture at execution time, without privileged state, auxiliary inference, or additional communication.

\subsection{Training Algorithm}
\label{app:spl-algorithm}

Algorithm~\ref{alg:spl} summarizes one SPL training iteration. SPL preserves the native optimization procedure of each MARL backbone and adds the role-appropriate auxiliary actor update. Earlier participants receive prospective acquisition supervision, while future collaborators receive receiver supervision once the corresponding replay-readiness condition is satisfied.

\begin{algorithm}[H]
\footnotesize
\caption{One training iteration with Staggered Participation Learning}
\label{alg:spl}
\begin{algorithmic}[1]

\Require Backbone $\mathcal{B}$; policies
$\{\pi_{\theta_i}\}_{i\in\mathcal{N}}$;
frozen acquisition model;
online outcome model $p_\psi$;
target outcome model $p_{\bar{\psi}}$;
outcome replay $\mathcal{D}_{\mathrm{out}}$

\State Collect one cooperative rollout
\State Compute native return and advantage inputs
\State Extract attributable future-collaborator outcomes
\State Add retained samples $(x_t^j,u_t^j,j,y_t^j)$ to
$\mathcal{D}_{\mathrm{out}}$

\If{outcome-model readiness condition is satisfied}
    \State Update $\psi$ using $\mathcal{L}_{\mathrm{out}}$
    \State $\bar{\psi}\gets(1-\rho)\bar{\psi}+\rho\psi$
\EndIf

\State Construct acquisition targets for eligible earlier participants

\If{receiver-supervision readiness condition is satisfied}
    \State Construct receiver targets using $p_{\bar{\psi}}$
\EndIf

\If{$\mathcal{B}$ is MAPPO}
    \For{each actor $i$}
        \State Apply the native MAPPO actor update to $i$
        \State Apply the eligible SPL actor update to $i$
    \EndFor
    \State Update the centralized critic

\ElsIf{$\mathcal{B}$ is IPPO}
    \For{each actor $i$}
        \State Apply the native IPPO actor update to $i$
        \State Apply the eligible SPL actor update to $i$
        \State Update the local critic of $i$
    \EndFor

\ElsIf{$\mathcal{B}$ is HAPPO}
    \State Initialize cumulative factor $F\gets1$
    \State Select the HAPPO actor-update order

    \For{each actor $i$ in the selected order}
        \State Store the current factor $F$ for actor $i$
        \State Evaluate the pre-update action probabilities
        \State Apply the native HAPPO actor update to $i$
        \State Apply the eligible SPL actor update to $i$
        \State Evaluate the post-update action probabilities
        \State Update the cumulative factor $F$
    \EndFor

    \State Update the centralized critic
\EndIf

\State \Return updated policies, critics, outcome models, and replay

\end{algorithmic}
\end{algorithm}
\FloatBarrier

\subsection{Eight-Agent Scaling Protocol}
\label{app:eight-agent-scaling}

\paragraph{Experimental scope.}
The eight-agent scaling study evaluates MPE-Advanced and RWARE-Advanced with four earlier participants $P_0,\ldots,P_3$, four future collaborators $R_0,\ldots,R_3$, and sixteen candidate tasks. We evaluate Non-overlap, Dynamic, and Cascaded participation using MAPPO and SPL-MAPPO. Non-overlap and Dynamic are trained for two million environment steps, while Cascaded uses four million steps to accommodate its extended participation horizon; all configurations use three independent training runs.

\begin{table}[!htbp]
\centering
\caption{Task and evidence configuration for the eight-agent scaling experiments. The notation $x^{\times4}$ denotes four repetitions of value $x$.}
\label{tab:app-eight-agent-task-settings}
\small
\setlength{\tabcolsep}{4.5pt}
\renewcommand{\arraystretch}{1.06}
\begin{adjustbox}{max width=\linewidth}
\begin{tabular}{@{}lcc@{}}
\toprule
Setting & MPE-Advanced & RWARE-Advanced \\
\midrule
Earlier participants & 4 & 4 \\
Future collaborators & 4 & 4 \\
Candidate tasks & 16 & 16 \\
Fixed valid tasks & 8 & 8 \\
Valid tasks / four-task group & 2 & 2 \\
Validity priors
& \multicolumn{2}{c}{
$(0.80^{\times4},0.65^{\times4},0.35^{\times4},0.20^{\times4})$} \\
Probe budget / earlier participant & 5 & 6 \\
Probe limit / task & 2 & 2 \\
Probe accuracy & 0.85 & 0.85 \\
Publication delay & 5 & 5 \\
Execution budget / future collaborator & 2 & 2 \\
Valid-task value
& $(12^{\times4},10^{\times4},8^{\times4},6^{\times4})$
& 1 \\
Probe cost & 0.20 & 0.01 \\
Invalid-execution cost & 8.00 & 0.25 \\
\bottomrule
\end{tabular}
\end{adjustbox}
\end{table}

\paragraph{Environment configuration.}
The sixteen candidate tasks form four groups of four, with two valid tasks sampled in each group. In MPE-Advanced, balanced task shards are randomly assigned at reset and the earlier-participant--future-collaborator correspondence is randomly assigned at handoff. This correspondence is not available to earlier participants beforehand. Earlier participants and future collaborators spawn on opposite sides of the environment at $x=-1.82$ and $x=1.82$, respectively, with
\[
y\in\{0.76,\ 0.2533,\ -0.2533,\ -0.76\}.
\]
The agent speed, view radius, interaction radius, and placement jitter are $1.0$, $4.0$, $0.18$, and $0.02$. Each actor observation has dimension 339, and the MAPPO critic input concatenates the eight actor observations, giving dimension 2712. Hidden task validity is unavailable to both representations.

RWARE-Advanced uses one shelf row, three shelf columns of height eight, and sensor range one. Candidate task identifiers are permuted each episode, future collaborators select from a shared candidate pool, and competing claims are resolved by first arrival. Hidden task validity is excluded from both actor observations and the MAPPO critic input.

\paragraph{Temporal participation schedules.}
All participation intervals follow the half-open convention $[t_{\mathrm{start}},t_{\mathrm{end}})$. Let $H$ denote the ordinary episode horizon, $a$ the nominal handoff time, $d$ the maximum Dynamic deviation, and $H_{\mathrm{cas}}$ the Cascaded horizon. Table~\ref{tab:app-eight-agent-temporal-parameters} summarizes the corresponding values.

\begin{table}[!htbp]
\centering
\caption{Temporal parameters and aggregate role exposure for the eight-agent experiments. Exposure is measured in agent-steps per episode.}
\label{tab:app-eight-agent-temporal-parameters}
\small
\setlength{\tabcolsep}{4.5pt}
\renewcommand{\arraystretch}{1.06}
\begin{adjustbox}{max width=\linewidth}
\begin{tabular}{@{}lrrrrrr@{}}
\toprule
Domain
& $H$
& $a$
& $d$
& $H_{\mathrm{cas}}$
& Earlier
& Future \\
\midrule
MPE-Advanced   & 125 & 45  & 11 & 500  & 180 & 320 \\
RWARE-Advanced & 500 & 200 & 50 & 2000 & 800 & 1200 \\
\bottomrule
\end{tabular}
\end{adjustbox}
\end{table}

For Non-overlap,
\begin{equation}
\mathcal{I}^{P_i}_{\mathrm{non}}=[0,a),
\qquad
\mathcal{I}^{R_i}_{\mathrm{non}}=[a,H),
\qquad
i\in\{0,1,2,3\}.
\label{eq:app-eight-agent-nonoverlap}
\end{equation}

For Dynamic participation, two independent deviations are sampled per episode,
\begin{equation}
\delta_k
\sim
\operatorname{Unif}\{-d,\ldots,d\},
\qquad
k\in\{0,1\},
\label{eq:app-eight-agent-dynamic-deviation}
\end{equation}
with paired handoffs
\begin{equation}
h_{2k}=a+\delta_k,
\qquad
h_{2k+1}=a-\delta_k.
\label{eq:app-eight-agent-dynamic-handoffs}
\end{equation}
The corresponding intervals are
\begin{equation}
\mathcal{I}^{P_i}_{\mathrm{dyn}}=[0,h_i),
\qquad
\mathcal{I}^{R_i}_{\mathrm{dyn}}=[h_i,H).
\label{eq:app-eight-agent-dynamic-windows}
\end{equation}
The handoff vector is sampled each episode and included in the actor-visible schedule information. Since $\sum_{i=0}^{3}h_i=4a$, Dynamic preserves the same aggregate role exposure as Non-overlap.

For Cascaded participation,
\begin{equation}
\mathcal{I}^{P_i}_{\mathrm{cas}}=[iH,iH+a),
\qquad
\mathcal{I}^{R_i}_{\mathrm{cas}}=[iH+a,(i+1)H),
\qquad
i\in\{0,1,2,3\}.
\label{eq:app-eight-agent-cascaded-windows}
\end{equation}
The resulting horizon is $H_{\mathrm{cas}}=4H$, preserving aggregate role exposure while distributing the four earlier-participant--future-collaborator pairs across successive intervals.

\paragraph{Training and evaluation.}
The eight-agent MAPPO and SPL-MAPPO runs use the feed-forward architecture, PPO settings, and SPL-specific hyperparameters reported in Table~\ref{tab:app-hyperparameters}. Policies are not parameter-shared and recurrent policies are disabled. Non-overlap and Dynamic are trained for two million environment steps, whereas Cascaded is trained for four million steps; each configuration uses three independent runs and is evaluated over 200 final episodes per run. MPE-Advanced uses eight parallel rollout environments, whereas RWARE-Advanced uses one. Reported task-completion results use the final checkpoint at the corresponding training budget, without performance-based checkpoint selection.

\section{Observed Training Cost}
\label{app:training-cost}

We report observed wall-clock training time as a practical indication of computational cost. Measurements correspond to completed full-budget runs; short qualification and smoke-scale runs are excluded. Because experiments ran under varying system load and host hardware was not recorded consistently, these values are descriptive runtime measurements. They are not controlled hardware-throughput benchmarks or estimates of algorithmic complexity.

Table~\ref{tab:app-training-time} summarizes the observed training time. In the primary four-agent experiments, all methods use two million environment steps, and Table~\ref{tab:app-training-time}(a) reports mean minutes per completed run averaged over the ten settings formed by two environment variants and five participation patterns. Table~\ref{tab:app-training-time}(b) reports the eight-agent scaling and Isaac Lab experiments. The eight-agent Non-overlap and Dynamic settings use 2M environment steps, whereas Cascaded uses 4M. The two- and four-agent Isaac Lab settings use approximately 2.0M and 3.0M steps, respectively.

\begin{table}[!htbp]
\centering
\caption{Observed wall-clock training time in minutes. Primary four-agent values are averages across the ten environment--participation settings in each domain; scaling and Isaac Lab values are mean $\pm$ standard error over three independent runs. Light blue shading identifies SPL results.}
\label{tab:app-training-time}

\begin{subtable}[t]{0.33\textwidth}
\centering
\vspace{0pt}
\caption{Primary four-agent experiments.}
\label{tab:app-four-agent-training-time}
\small
\setlength{\tabcolsep}{7pt}
\renewcommand{\arraystretch}{1.06}
\begin{adjustbox}{max width=\linewidth}
\begin{tabular}{lrr}
\toprule
Method & MPE & RWARE \\
\midrule
IPPO       & 154.1 & 166.4 \\
\rowcolor{splbg}
SPL-IPPO   & 175.3 & 177.9 \\
\addlinespace[2pt]
MAPPO      & 114.0 & 128.3 \\
\rowcolor{splbg}
SPL-MAPPO  & 131.1 & 129.6 \\
\addlinespace[2pt]
HAPPO      & 134.6 & 128.5 \\
\rowcolor{splbg}
SPL-HAPPO  & 152.4 & 145.2 \\
\midrule
MAE-MAPPO  & 157.2 & 128.6 \\
Sable      & 205.1 & 210.8 \\
SMPE$^2$   & 119.2 & 218.0 \\
\bottomrule
\end{tabular}
\end{adjustbox}
\end{subtable}
\hfill
\begin{subtable}[t]{0.64\textwidth}
\centering
\vspace{0pt}
\caption{Scaling and Isaac Lab experiments.}
\label{tab:app-additional-training-time}
\small
\setlength{\tabcolsep}{4.5pt}
\renewcommand{\arraystretch}{1.06}
\begin{adjustbox}{max width=\linewidth}
\begin{tabular}{llrcQ}
\toprule
Setting & Participation & Steps & MAPPO & SPL-MAPPO \\
\midrule
\multirow[t]{3}{*}{MPE-Adv. (8)}
& Non-overlap & 2.00M & $76.0 \pm 0.2$    & $87.0 \pm 0.1$ \\
& Dynamic     & 2.00M & $89.5 \pm 0.6$    & $103.0 \pm 0.5$ \\
& Cascaded    & 4.00M & $175.5 \pm 0.2$   & $185.5 \pm 2.7$ \\
\addlinespace[2pt]

\multirow[t]{3}{*}{RWARE-Adv. (8)}
& Non-overlap & 2.00M & $242.9 \pm 6.1$   & $396.9 \pm 15.3$ \\
& Dynamic     & 2.00M & $216.6 \pm 3.6$   & $377.5 \pm 11.0$ \\
& Cascaded    & 4.00M & $453.6 \pm 16.6$  & $538.3 \pm 15.4$ \\
\addlinespace[2pt]

\multirow[t]{2}{*}{Isaac Lab (2)}
& Non-overlap & 2.00M & $548.8 \pm 17.5$  & $562.9 \pm 18.8$ \\
& Dynamic     & 2.00M & $438.1 \pm 9.9$   & $450.0 \pm 7.0$ \\
\addlinespace[2pt]

\multirow[t]{2}{*}{Isaac Lab (4)}
& Non-overlap & 3.01M & $1090.2 \pm 17.8$ & $1244.7 \pm 43.2$ \\
& Dynamic     & 3.01M & $1120.1 \pm 15.1$ & $1295.0 \pm 7.3$ \\
\bottomrule
\end{tabular}
\end{adjustbox}
\end{subtable}

\end{table}

\paragraph{Observation.}
The practical training-time effect of SPL varies across environments and team configurations. In the primary four-agent experiments, SPL incurs moderate additional wall-clock cost for most PPO backbones, while MAPPO and SPL-MAPPO are nearly identical in RWARE. In the eight-agent experiments, the relative increase is modest in MPE and larger in RWARE; absolute times across participation patterns should be interpreted together with their different training budgets, particularly the 4M-step Cascaded settings. The two-agent Isaac Lab runs show only a small difference, whereas the four-agent settings require approximately 14--16\% more observed training time with SPL-MAPPO. Overall, SPL introduces additional training-time computation whose magnitude depends on the environment and configuration, while decentralized execution remains unchanged.

\section{Detailed Experiment Results}
\label{app:detailed-results}

The following tables provide the complete evaluation metrics underlying the main-paper results. Task completion (TC) measures the percentage of required downstream task objectives successfully completed within an evaluation episode, while full success (FS) is the percentage of episodes in which all required downstream tasks are completed. We also report episode return and failed execution (FE). Higher values are better for TC, FS, and return, whereas lower FE is better. MPE and RWARE values are reported as mean $\pm$ standard error over five independent training seeds and 200 final evaluation episodes per seed. A reported standard error of $0.00$ indicates that the five seed-level evaluation means are exactly identical at full precision; further discussion is provided in Appendix~D.7.

\subsection{MPE-Base}
\label{app:results-mpe-base}

Table~\ref{tab:appendix-mpe-base-full} reports the complete MPE-Base results across the five participation patterns, including task completion, full success, return, and failed execution.

\begin{table}[!htbp]
\centering
\caption{Complete MPE-Base four-agent results at 2M environment steps. Values are mean $\pm$ standard error over five independent training seeds and 200 final evaluation episodes per seed. Bold marks the best reported mean in each participation--metric row. For legacy MPE-Base Non-overlap HARL runs, FS is recovered from the retained per-episode completed-task counts using the same all-four-tasks definition.}
\label{tab:appendix-mpe-base-full}
\scriptsize
\setlength{\tabcolsep}{2.2pt}
\renewcommand{\arraystretch}{1.02}
\resizebox{\textwidth}{!}{%
\begin{adjustbox}{max width=\linewidth}
\begin{tabular}{llcQcQcQccc}
\toprule
\multirow{2}{*}{Participation}
& \multirow{2}{*}{Metric}
& \multicolumn{2}{c}{IPPO}
& \multicolumn{2}{c}{MAPPO}
& \multicolumn{2}{c}{HAPPO}
& \multicolumn{3}{c}{Other baselines} \\
\cmidrule(lr){3-4}
\cmidrule(lr){5-6}
\cmidrule(lr){7-8}
\cmidrule(lr){9-11}
& & Vanilla & SPL & Vanilla & SPL & Vanilla & SPL & MAE-MAPPO & Sable & SMPE$^2$ \\
\midrule
\multirow{4}{*}{Non-overlap}
& TC (\%) $\uparrow$ & $70.50 \pm 0.00$ & $78.38 \pm 0.38$ & $70.50 \pm 0.00$ & $78.73 \pm 0.43$ & $67.45 \pm 3.05$ & $\mathbf{79.18 \pm 0.23}$ & $70.50 \pm 0.00$ & $62.70 \pm 0.25$ & $42.73 \pm 9.29$ \\
& FS (\%) $\uparrow$ & $17.50 \pm 0.00$ & $31.80 \pm 0.87$ & $17.50 \pm 0.00$ & $32.80 \pm 1.06$ & $14.00 \pm 3.50$ & $\mathbf{34.20 \pm 0.49}$ & $17.50 \pm 0.00$ & $9.00 \pm 0.71$ & $4.20 \pm 2.57$ \\
& Return $\uparrow$ & $21.569 \pm 0.029$ & $27.337 \pm 0.083$ & $21.644 \pm 0.036$ & $27.231 \pm 0.091$ & $21.010 \pm 0.675$ & $\mathbf{27.409 \pm 0.048}$ & $21.747 \pm 0.076$ & $15.100 \pm 0.199$ & $11.016 \pm 2.409$ \\
& FE $\downarrow$ & $1.125 \pm 0.000$ & $\mathbf{0.582 \pm 0.008}$ & $1.127 \pm 0.002$ & $0.610 \pm 0.012$ & $1.056 \pm 0.069$ & $0.608 \pm 0.004$ & $1.125 \pm 0.000$ & $1.450 \pm 0.011$ & $0.873 \pm 0.232$ \\
\cmidrule(lr){1-11}
\multirow{4}{*}{Overlap-25}
& TC (\%) $\uparrow$ & $68.25 \pm 2.75$ & $79.35 \pm 0.57$ & $70.78 \pm 0.14$ & $80.25 \pm 0.21$ & $70.70 \pm 0.18$ & $\mathbf{80.35 \pm 0.29}$ & $67.93 \pm 2.95$ & $63.50 \pm 2.24$ & $53.85 \pm 5.43$ \\
& FS (\%) $\uparrow$ & $14.90 \pm 2.60$ & $34.10 \pm 1.45$ & $17.50 \pm 0.00$ & $\mathbf{36.80 \pm 0.46}$ & $17.50 \pm 0.16$ & $36.40 \pm 0.81$ & $14.00 \pm 3.38$ & $9.30 \pm 2.22$ & $4.30 \pm 2.53$ \\
& Return $\uparrow$ & $20.308 \pm 1.009$ & $27.068 \pm 0.212$ & $21.521 \pm 0.082$ & $27.232 \pm 0.139$ & $21.439 \pm 0.174$ & $\mathbf{27.358 \pm 0.122}$ & $20.950 \pm 0.554$ & $15.295 \pm 1.868$ & $12.778 \pm 2.124$ \\
& FE $\downarrow$ & $1.124 \pm 0.036$ & $\mathbf{0.654 \pm 0.009}$ & $1.161 \pm 0.007$ & $0.670 \pm 0.014$ & $1.159 \pm 0.010$ & $0.658 \pm 0.017$ & $1.079 \pm 0.079$ & $1.460 \pm 0.090$ & $1.246 \pm 0.236$ \\
\cmidrule(lr){1-11}
\multirow{4}{*}{Overlap-50}
& TC (\%) $\uparrow$ & $67.85 \pm 3.15$ & $\mathbf{78.90 \pm 0.31}$ & $69.53 \pm 1.27$ & $\mathbf{78.90 \pm 0.24}$ & $69.80 \pm 0.92$ & $78.50 \pm 0.47$ & $70.63 \pm 0.23$ & $64.43 \pm 1.85$ & $63.95 \pm 3.34$ \\
& FS (\%) $\uparrow$ & $14.00 \pm 3.50$ & $\mathbf{32.20 \pm 0.60}$ & $15.70 \pm 1.36$ & $31.90 \pm 0.71$ & $16.20 \pm 0.97$ & $30.70 \pm 1.11$ & $17.50 \pm 0.16$ & $10.30 \pm 2.12$ & $10.60 \pm 3.20$ \\
& Return $\uparrow$ & $20.744 \pm 0.680$ & $26.429 \pm 0.112$ & $21.213 \pm 0.287$ & $\mathbf{26.569 \pm 0.098}$ & $21.228 \pm 0.340$ & $26.125 \pm 0.134$ & $21.353 \pm 0.166$ & $15.888 \pm 1.503$ & $15.529 \pm 2.966$ \\
& FE $\downarrow$ & $1.086 \pm 0.074$ & $0.731 \pm 0.003$ & $1.136 \pm 0.033$ & $\mathbf{0.715 \pm 0.008}$ & $1.148 \pm 0.009$ & $0.748 \pm 0.003$ & $1.161 \pm 0.010$ & $1.422 \pm 0.074$ & $1.442 \pm 0.133$ \\
\cmidrule(lr){1-11}
\multirow{4}{*}{Dynamic}
& TC (\%) $\uparrow$ & $66.20 \pm 3.01$ & $\mathbf{79.65 \pm 0.18}$ & $70.23 \pm 0.59$ & $78.88 \pm 0.18$ & $69.53 \pm 0.99$ & $79.63 \pm 0.39$ & $64.85 \pm 3.06$ & $58.30 \pm 1.11$ & $49.70 \pm 1.04$ \\
& FS (\%) $\uparrow$ & $12.60 \pm 3.43$ & $35.30 \pm 0.51$ & $16.70 \pm 0.80$ & $33.60 \pm 0.43$ & $16.30 \pm 1.20$ & $\mathbf{35.60 \pm 1.14}$ & $11.20 \pm 3.39$ & $5.30 \pm 0.58$ & $0.60 \pm 0.60$ \\
& Return $\uparrow$ & $20.694 \pm 0.600$ & $\mathbf{27.643 \pm 0.090}$ & $21.364 \pm 0.232$ & $27.306 \pm 0.124$ & $21.153 \pm 0.367$ & $27.499 \pm 0.146$ & $20.613 \pm 0.613$ & $13.970 \pm 0.839$ & $12.579 \pm 1.591$ \\
& FE $\downarrow$ & $1.003 \pm 0.079$ & $\mathbf{0.599 \pm 0.009}$ & $1.125 \pm 0.010$ & $0.610 \pm 0.012$ & $1.119 \pm 0.015$ & $0.610 \pm 0.014$ & $0.963 \pm 0.077$ & $1.318 \pm 0.038$ & $1.022 \pm 0.139$ \\
\cmidrule(lr){1-11}
\multirow{4}{*}{Cascaded}
& TC (\%) $\uparrow$ & $66.83 \pm 3.68$ & $78.90 \pm 0.70$ & $70.63 \pm 0.10$ & $\mathbf{78.93 \pm 0.48}$ & $70.68 \pm 0.19$ & $78.78 \pm 0.22$ & $70.45 \pm 0.05$ & $38.03 \pm 0.93$ & $27.93 \pm 5.74$ \\
& FS (\%) $\uparrow$ & $14.00 \pm 3.50$ & $\mathbf{33.70 \pm 1.79}$ & $17.50 \pm 0.00$ & $33.50 \pm 0.85$ & $17.40 \pm 0.10$ & $33.00 \pm 0.57$ & $17.50 \pm 0.00$ & $0.50 \pm 0.32$ & $0.20 \pm 0.20$ \\
& Return $\uparrow$ & $20.872 \pm 0.753$ & $27.563 \pm 0.259$ & $21.529 \pm 0.060$ & $27.447 \pm 0.207$ & $21.502 \pm 0.048$ & $\mathbf{27.581 \pm 0.059}$ & $21.642 \pm 0.047$ & $9.745 \pm 0.861$ & $6.875 \pm 2.668$ \\
& FE $\downarrow$ & $1.040 \pm 0.085$ & $\mathbf{0.577 \pm 0.015}$ & $1.138 \pm 0.004$ & $0.590 \pm 0.018$ & $1.144 \pm 0.008$ & $0.578 \pm 0.014$ & $1.127 \pm 0.002$ & $0.850 \pm 0.058$ & $0.650 \pm 0.059$ \\
\bottomrule
\end{tabular}
\end{adjustbox}%
}
\end{table}
\FloatBarrier

\paragraph{Observation.}
Across all five participation patterns, SPL improves TC, FS, and return while reducing FE relative to the matched Vanilla backbone. The improvement is consistent across IPPO, MAPPO, and HAPPO, indicating that the additional gains are not specific to a single PPO-based backbone.

\subsection{MPE-Advanced}
\label{app:results-mpe-advanced}

Table~\ref{tab:appendix-mpe-advanced-full} reports the corresponding results in MPE-Advanced, where task assignments, spawn configurations, and future-collaborator identities vary across episodes.

\begin{table}[!htbp]
\centering
\caption{Complete MPE-Advanced four-agent results at 2M environment steps. Values are mean $\pm$ standard error over five independent training seeds and 200 final evaluation episodes per seed. Bold marks the best reported mean in each participation--metric row.}
\label{tab:appendix-mpe-advanced-full}
\scriptsize
\setlength{\tabcolsep}{2.2pt}
\renewcommand{\arraystretch}{1.02}
\resizebox{\textwidth}{!}{%
\begin{adjustbox}{max width=\linewidth}
\begin{tabular}{llcQcQcQccc}
\toprule
\multirow{2}{*}{Participation}
& \multirow{2}{*}{Metric}
& \multicolumn{2}{c}{IPPO}
& \multicolumn{2}{c}{MAPPO}
& \multicolumn{2}{c}{HAPPO}
& \multicolumn{3}{c}{Other baselines} \\
\cmidrule(lr){3-4}
\cmidrule(lr){5-6}
\cmidrule(lr){7-8}
\cmidrule(lr){9-11}
& & Vanilla & SPL & Vanilla & SPL & Vanilla & SPL & MAE-MAPPO & Sable & SMPE$^2$ \\
\midrule
\multirow{4}{*}{Non-overlap}
& TC (\%) $\uparrow$ & $70.38 \pm 0.00$ & $77.95 \pm 0.19$ & $64.08 \pm 3.86$ & $\mathbf{78.23 \pm 0.16}$ & $67.28 \pm 3.10$ & $78.03 \pm 0.16$ & $63.28 \pm 3.47$ & $64.90 \pm 1.36$ & $49.78 \pm 3.25$ \\
& FS (\%) $\uparrow$ & $19.00 \pm 0.00$ & $30.50 \pm 0.65$ & $11.40 \pm 4.65$ & $\mathbf{31.90 \pm 0.62}$ & $15.20 \pm 3.80$ & $31.10 \pm 0.53$ & $10.60 \pm 4.35$ & $10.80 \pm 1.41$ & $1.60 \pm 1.09$ \\
& Return $\uparrow$ & $21.485 \pm 0.135$ & $27.229 \pm 0.092$ & $19.986 \pm 0.847$ & $\mathbf{27.390 \pm 0.106}$ & $20.754 \pm 0.645$ & $27.328 \pm 0.076$ & $19.939 \pm 0.798$ & $16.683 \pm 1.085$ & $10.730 \pm 2.269$ \\
& FE $\downarrow$ & $1.125 \pm 0.000$ & $0.589 \pm 0.008$ & $0.989 \pm 0.083$ & $\mathbf{0.579 \pm 0.007}$ & $1.056 \pm 0.069$ & $0.580 \pm 0.006$ & $0.971 \pm 0.078$ & $1.369 \pm 0.057$ & $1.238 \pm 0.256$ \\
\cmidrule(lr){1-11}
\multirow{4}{*}{Overlap-25}
& TC (\%) $\uparrow$ & $65.65 \pm 3.53$ & $\mathbf{79.95 \pm 0.52}$ & $68.28 \pm 1.50$ & $79.38 \pm 0.27$ & $70.93 \pm 0.56$ & $79.55 \pm 0.40$ & $71.40 \pm 0.23$ & $64.30 \pm 1.63$ & $47.83 \pm 9.02$ \\
& FS (\%) $\uparrow$ & $12.60 \pm 4.28$ & $\mathbf{35.50 \pm 0.57}$ & $16.60 \pm 1.27$ & $33.80 \pm 0.73$ & $18.70 \pm 0.80$ & $34.50 \pm 1.21$ & $19.30 \pm 0.20$ & $10.50 \pm 2.20$ & $5.20 \pm 4.17$ \\
& Return $\uparrow$ & $20.343 \pm 0.734$ & $\mathbf{27.536 \pm 0.235}$ & $20.611 \pm 0.494$ & $27.244 \pm 0.139$ & $21.634 \pm 0.205$ & $27.502 \pm 0.185$ & $21.967 \pm 0.120$ & $15.589 \pm 1.238$ & $13.152 \pm 2.718$ \\
& FE $\downarrow$ & $1.011 \pm 0.077$ & $0.622 \pm 0.012$ & $1.101 \pm 0.027$ & $0.641 \pm 0.007$ & $1.126 \pm 0.016$ & $\mathbf{0.611 \pm 0.012}$ & $1.131 \pm 0.004$ & $1.427 \pm 0.066$ & $0.901 \pm 0.255$ \\
\cmidrule(lr){1-11}
\multirow{4}{*}{Overlap-50}
& TC (\%) $\uparrow$ & $71.30 \pm 0.22$ & $78.63 \pm 0.23$ & $68.58 \pm 3.05$ & $78.33 \pm 0.40$ & $71.40 \pm 0.14$ & $\mathbf{78.68 \pm 0.25}$ & $62.63 \pm 5.92$ & $61.88 \pm 1.48$ & $52.08 \pm 8.25$ \\
& FS (\%) $\uparrow$ & $19.20 \pm 0.30$ & $31.90 \pm 0.60$ & $15.80 \pm 3.70$ & $31.10 \pm 1.11$ & $19.40 \pm 0.10$ & $\mathbf{32.10 \pm 0.76}$ & $12.00 \pm 4.60$ & $9.10 \pm 1.20$ & $6.10 \pm 4.23$ \\
& Return $\uparrow$ & $21.669 \pm 0.083$ & $\mathbf{26.481 \pm 0.084}$ & $21.005 \pm 0.701$ & $26.358 \pm 0.171$ & $21.567 \pm 0.090$ & $26.446 \pm 0.097$ & $19.743 \pm 1.443$ & $13.716 \pm 1.266$ & $13.018 \pm 3.412$ \\
& FE $\downarrow$ & $1.124 \pm 0.007$ & $\mathbf{0.713 \pm 0.007}$ & $1.069 \pm 0.066$ & $0.715 \pm 0.010$ & $1.143 \pm 0.006$ & $0.721 \pm 0.005$ & $0.958 \pm 0.115$ & $1.525 \pm 0.059$ & $1.139 \pm 0.067$ \\
\cmidrule(lr){1-11}
\multirow{4}{*}{Dynamic}
& TC (\%) $\uparrow$ & $66.63 \pm 3.66$ & $78.23 \pm 0.29$ & $67.05 \pm 3.20$ & $\mathbf{79.08 \pm 0.54}$ & $61.73 \pm 5.55$ & $78.80 \pm 0.18$ & $67.68 \pm 1.76$ & $57.60 \pm 1.30$ & $43.10 \pm 6.61$ \\
& FS (\%) $\uparrow$ & $14.70 \pm 3.68$ & $32.30 \pm 0.64$ & $14.70 \pm 3.68$ & $34.10 \pm 1.38$ & $11.10 \pm 4.53$ & $\mathbf{34.20 \pm 0.64}$ & $15.60 \pm 2.42$ & $6.50 \pm 1.10$ & $0.00 \pm 0.00$ \\
& Return $\uparrow$ & $20.509 \pm 0.893$ & $27.006 \pm 0.131$ & $20.770 \pm 0.681$ & $27.312 \pm 0.258$ & $19.867 \pm 1.215$ & $\mathbf{27.330 \pm 0.081}$ & $20.936 \pm 0.420$ & $14.027 \pm 1.018$ & $13.678 \pm 2.577$ \\
& FE $\downarrow$ & $1.037 \pm 0.072$ & $0.608 \pm 0.005$ & $1.047 \pm 0.064$ & $0.608 \pm 0.013$ & $0.915 \pm 0.123$ & $\mathbf{0.596 \pm 0.012}$ & $1.047 \pm 0.038$ & $1.272 \pm 0.052$ & $0.608 \pm 0.083$ \\
\cmidrule(lr){1-11}
\multirow{4}{*}{Cascaded}
& TC (\%) $\uparrow$ & $67.68 \pm 2.58$ & $76.80 \pm 0.13$ & $67.28 \pm 3.10$ & $77.00 \pm 0.33$ & $69.28 \pm 1.10$ & $\mathbf{77.23 \pm 0.18}$ & $66.83 \pm 3.02$ & $62.73 \pm 0.98$ & $56.70 \pm 5.50$ \\
& FS (\%) $\uparrow$ & $15.90 \pm 3.10$ & $29.00 \pm 0.45$ & $15.20 \pm 3.80$ & $29.20 \pm 0.68$ & $17.80 \pm 1.20$ & $\mathbf{30.70 \pm 0.72}$ & $14.80 \pm 3.72$ & $8.30 \pm 1.66$ & $7.10 \pm 3.72$ \\
& Return $\uparrow$ & $20.895 \pm 0.638$ & $25.976 \pm 0.072$ & $20.623 \pm 0.661$ & $25.988 \pm 0.143$ & $21.038 \pm 0.238$ & $\mathbf{26.033 \pm 0.100}$ & $20.707 \pm 0.662$ & $14.706 \pm 0.811$ & $18.650 \pm 1.423$ \\
& FE $\downarrow$ & $1.065 \pm 0.055$ & $\mathbf{0.734 \pm 0.006}$ & $1.056 \pm 0.069$ & $0.742 \pm 0.007$ & $1.101 \pm 0.024$ & $0.741 \pm 0.007$ & $1.043 \pm 0.067$ & $1.457 \pm 0.035$ & $0.795 \pm 0.108$ \\
\bottomrule
\end{tabular}
\end{adjustbox}%
}
\end{table}
\FloatBarrier

\paragraph{Observation.}
The same pattern persists under the additional structural variation of MPE-Advanced. SPL improves all four reported metrics for every matched backbone--participation comparison, while the three SPL backbones remain close in final performance.

\subsection{RWARE-Base}
\label{app:results-rware-base}

Table~\ref{tab:appendix-rware-base-full} reports the complete RWARE-Base results across the five participation patterns.

\begin{table}[!htbp]
\centering
\caption{Complete RWARE-Base four-agent results at 2M environment steps. Values are mean $\pm$ standard error over five independent training seeds and 200 final evaluation episodes per seed. Bold marks the best reported mean in each participation--metric row.}
\label{tab:appendix-rware-base-full}
\scriptsize
\setlength{\tabcolsep}{2.2pt}
\renewcommand{\arraystretch}{1.02}
\resizebox{\textwidth}{!}{%
\begin{adjustbox}{max width=\linewidth}
\begin{tabular}{llcQcQcQccc}
\toprule
\multirow{2}{*}{Participation}
& \multirow{2}{*}{Metric}
& \multicolumn{2}{c}{IPPO}
& \multicolumn{2}{c}{MAPPO}
& \multicolumn{2}{c}{HAPPO}
& \multicolumn{3}{c}{Other baselines} \\
\cmidrule(lr){3-4}
\cmidrule(lr){5-6}
\cmidrule(lr){7-8}
\cmidrule(lr){9-11}
& & Vanilla & SPL & Vanilla & SPL & Vanilla & SPL & MAE-MAPPO & Sable & SMPE$^2$ \\
\midrule
\multirow{4}{*}{Non-overlap}
& TC (\%) $\uparrow$ & $66.45 \pm 5.86$ & $86.83 \pm 0.40$ & $71.38 \pm 2.27$ & $86.70 \pm 0.49$ & $73.18 \pm 1.50$ & $\mathbf{86.93 \pm 0.26}$ & $73.48 \pm 1.69$ & $71.80 \pm 0.98$ & $66.05 \pm 3.29$ \\
& FS (\%) $\uparrow$ & $14.50 \pm 6.27$ & $53.80 \pm 1.68$ & $19.50 \pm 4.14$ & $53.70 \pm 1.46$ & $19.20 \pm 3.09$ & $\mathbf{54.00 \pm 0.57}$ & $21.30 \pm 2.28$ & $20.00 \pm 1.30$ & $14.10 \pm 3.61$ \\
& Return $\uparrow$ & $2.342 \pm 0.226$ & $3.292 \pm 0.018$ & $2.548 \pm 0.101$ & $3.286 \pm 0.023$ & $2.618 \pm 0.070$ & $\mathbf{3.297 \pm 0.010}$ & $2.641 \pm 0.072$ & $2.471 \pm 0.049$ & $2.301 \pm 0.164$ \\
& FE $\downarrow$ & $0.848 \pm 0.088$ & $0.494 \pm 0.013$ & $1.075 \pm 0.110$ & $0.503 \pm 0.019$ & $1.064 \pm 0.057$ & $\mathbf{0.483 \pm 0.011}$ & $1.003 \pm 0.068$ & $1.128 \pm 0.039$ & $1.358 \pm 0.131$ \\
\cmidrule(lr){1-11}
\multirow{4}{*}{Overlap-25}
& TC (\%) $\uparrow$ & $81.10 \pm 3.41$ & $\mathbf{89.15 \pm 0.16}$ & $74.25 \pm 3.41$ & $89.08 \pm 0.35$ & $70.15 \pm 2.93$ & $88.35 \pm 0.50$ & $72.50 \pm 2.86$ & $76.70 \pm 1.18$ & $67.28 \pm 2.87$ \\
& FS (\%) $\uparrow$ & $38.60 \pm 7.30$ & $\mathbf{61.00 \pm 0.32}$ & $25.90 \pm 7.18$ & $60.60 \pm 0.94$ & $16.60 \pm 4.20$ & $58.80 \pm 1.95$ & $19.20 \pm 5.35$ & $28.10 \pm 3.03$ & $16.00 \pm 2.83$ \\
& Return $\uparrow$ & $2.960 \pm 0.157$ & $\mathbf{3.404 \pm 0.008}$ & $2.667 \pm 0.157$ & $3.401 \pm 0.018$ & $2.500 \pm 0.131$ & $3.369 \pm 0.023$ & $2.633 \pm 0.105$ & $2.720 \pm 0.056$ & $2.358 \pm 0.145$ \\
& FE $\downarrow$ & $0.728 \pm 0.142$ & $0.411 \pm 0.011$ & $0.935 \pm 0.165$ & $\mathbf{0.410 \pm 0.018}$ & $1.075 \pm 0.092$ & $0.424 \pm 0.014$ & $0.877 \pm 0.114$ & $0.932 \pm 0.047$ & $1.309 \pm 0.115$ \\
\cmidrule(lr){1-11}
\multirow{4}{*}{Overlap-50}
& TC (\%) $\uparrow$ & $82.00 \pm 1.12$ & $\mathbf{89.00 \pm 0.19}$ & $74.73 \pm 3.35$ & $88.58 \pm 0.28$ & $70.28 \pm 2.41$ & $88.80 \pm 0.23$ & $77.53 \pm 0.72$ & $77.23 \pm 1.97$ & $61.95 \pm 4.09$ \\
& FS (\%) $\uparrow$ & $41.80 \pm 1.74$ & $59.40 \pm 0.68$ & $24.60 \pm 6.55$ & $59.30 \pm 0.94$ & $15.20 \pm 3.03$ & $\mathbf{59.70 \pm 0.80}$ & $30.20 \pm 1.32$ & $31.20 \pm 4.47$ & $10.10 \pm 4.15$ \\
& Return $\uparrow$ & $3.008 \pm 0.052$ & $\mathbf{3.397 \pm 0.010}$ & $2.706 \pm 0.153$ & $3.375 \pm 0.013$ & $2.487 \pm 0.112$ & $3.388 \pm 0.011$ & $2.834 \pm 0.033$ & $2.741 \pm 0.099$ & $2.143 \pm 0.175$ \\
& FE $\downarrow$ & $0.686 \pm 0.037$ & $\mathbf{0.414 \pm 0.010}$ & $0.979 \pm 0.144$ & $0.432 \pm 0.009$ & $1.184 \pm 0.099$ & $0.415 \pm 0.009$ & $0.858 \pm 0.039$ & $0.911 \pm 0.079$ & $1.322 \pm 0.150$ \\
\cmidrule(lr){1-11}
\multirow{4}{*}{Dynamic}
& TC (\%) $\uparrow$ & $72.88 \pm 1.19$ & $\mathbf{88.88 \pm 0.27}$ & $65.35 \pm 5.46$ & $88.70 \pm 0.30$ & $68.80 \pm 4.06$ & $88.40 \pm 0.42$ & $72.28 \pm 1.51$ & $74.63 \pm 1.22$ & $62.03 \pm 4.29$ \\
& FS (\%) $\uparrow$ & $19.30 \pm 1.56$ & $\mathbf{60.10 \pm 0.81}$ & $14.50 \pm 5.14$ & $58.70 \pm 1.20$ & $16.30 \pm 4.89$ & $58.40 \pm 1.37$ & $18.90 \pm 2.62$ & $24.70 \pm 2.72$ & $10.20 \pm 4.07$ \\
& Return $\uparrow$ & $2.579 \pm 0.041$ & $\mathbf{3.389 \pm 0.013}$ & $2.304 \pm 0.211$ & $3.385 \pm 0.014$ & $2.425 \pm 0.199$ & $3.368 \pm 0.020$ & $2.572 \pm 0.072$ & $2.616 \pm 0.058$ & $2.095 \pm 0.215$ \\
& FE $\downarrow$ & $0.910 \pm 0.118$ & $0.423 \pm 0.007$ & $1.119 \pm 0.129$ & $\mathbf{0.415 \pm 0.006}$ & $1.166 \pm 0.203$ & $0.438 \pm 0.014$ & $1.106 \pm 0.061$ & $1.015 \pm 0.049$ & $1.519 \pm 0.172$ \\
\cmidrule(lr){1-11}
\multirow{4}{*}{Cascaded}
& TC (\%) $\uparrow$ & $66.80 \pm 1.98$ & $86.25 \pm 0.38$ & $67.53 \pm 2.05$ & $\mathbf{86.75 \pm 0.50}$ & $66.58 \pm 2.24$ & $85.95 \pm 0.79$ & $71.63 \pm 0.73$ & $76.83 \pm 0.93$ & $64.18 \pm 2.94$ \\
& FS (\%) $\uparrow$ & $13.20 \pm 2.83$ & $52.40 \pm 1.07$ & $11.80 \pm 3.64$ & $\mathbf{53.60 \pm 1.02}$ & $11.90 \pm 2.64$ & $51.50 \pm 1.67$ & $18.70 \pm 0.97$ & $28.80 \pm 1.97$ & $11.00 \pm 3.67$ \\
& Return $\uparrow$ & $2.346 \pm 0.097$ & $3.281 \pm 0.014$ & $2.417 \pm 0.074$ & $\mathbf{3.299 \pm 0.023}$ & $2.322 \pm 0.110$ & $3.269 \pm 0.032$ & $2.573 \pm 0.031$ & $2.721 \pm 0.047$ & $2.207 \pm 0.146$ \\
& FE $\downarrow$ & $1.298 \pm 0.079$ & $0.456 \pm 0.012$ & $1.095 \pm 0.141$ & $0.466 \pm 0.011$ & $1.333 \pm 0.088$ & $\mathbf{0.448 \pm 0.008}$ & $1.135 \pm 0.029$ & $0.927 \pm 0.037$ & $1.433 \pm 0.118$ \\
\bottomrule
\end{tabular}
\end{adjustbox}%
}
\end{table}
\FloatBarrier
\paragraph{Observation.}
SPL produces larger gains in RWARE than in MPE, particularly in FS and FE. The increase in full success indicates that higher task completion is accompanied by more episodes in which the cooperative objective is completed in full.

\subsection{RWARE-Advanced}
\label{app:results-rware-advanced}

Table~\ref{tab:appendix-rware-advanced-full} reports the results in RWARE-Advanced, which additionally includes task permutation, a shared candidate pool, and dynamic task claims.

\begin{table}[!htbp]
\centering
\caption{Complete RWARE-Advanced four-agent results at 2M environment steps. Values are mean $\pm$ standard error over five independent training seeds and 200 final evaluation episodes per seed. Bold marks the best reported mean in each participation--metric row.}
\label{tab:appendix-rware-advanced-full}
\scriptsize
\setlength{\tabcolsep}{2.2pt}
\renewcommand{\arraystretch}{1.02}
\resizebox{\textwidth}{!}{%
\begin{adjustbox}{max width=\linewidth}
\begin{tabular}{llcQcQcQccc}
\toprule
\multirow{2}{*}{Participation}
& \multirow{2}{*}{Metric}
& \multicolumn{2}{c}{IPPO}
& \multicolumn{2}{c}{MAPPO}
& \multicolumn{2}{c}{HAPPO}
& \multicolumn{3}{c}{Other baselines} \\
\cmidrule(lr){3-4}
\cmidrule(lr){5-6}
\cmidrule(lr){7-8}
\cmidrule(lr){9-11}
& & Vanilla & SPL & Vanilla & SPL & Vanilla & SPL & MAE-MAPPO & Sable & SMPE$^2$ \\
\midrule
\multirow{4}{*}{Non-overlap}
& TC (\%) $\uparrow$ & $52.08 \pm 1.37$ & $79.93 \pm 0.59$ & $58.63 \pm 2.31$ & $\mathbf{80.03 \pm 0.23}$ & $57.58 \pm 3.44$ & $79.55 \pm 0.79$ & $58.50 \pm 2.20$ & $50.70 \pm 0.32$ & $45.70 \pm 0.25$ \\
& FS (\%) $\uparrow$ & $2.60 \pm 0.58$ & $40.60 \pm 0.87$ & $4.70 \pm 0.93$ & $\mathbf{42.20 \pm 1.44}$ & $5.70 \pm 1.79$ & $38.70 \pm 1.71$ & $4.30 \pm 1.19$ & $1.60 \pm 0.19$ & $1.20 \pm 0.62$ \\
& Return $\uparrow$ & $1.654 \pm 0.059$ & $2.950 \pm 0.026$ & $1.972 \pm 0.096$ & $\mathbf{2.951 \pm 0.010}$ & $1.891 \pm 0.144$ & $2.928 \pm 0.030$ & $1.951 \pm 0.099$ & $1.510 \pm 0.017$ & $1.346 \pm 0.022$ \\
& FE $\downarrow$ & $1.573 \pm 0.043$ & $\mathbf{0.577 \pm 0.013}$ & $1.174 \pm 0.098$ & $0.583 \pm 0.011$ & $1.386 \pm 0.115$ & $0.599 \pm 0.020$ & $1.330 \pm 0.119$ & $1.660 \pm 0.019$ & $1.882 \pm 0.025$ \\
\cmidrule(lr){1-11}
\multirow{4}{*}{Overlap-25}
& TC (\%) $\uparrow$ & $64.18 \pm 1.62$ & $83.23 \pm 0.47$ & $68.83 \pm 1.40$ & $83.65 \pm 0.32$ & $63.28 \pm 1.98$ & $\mathbf{83.73 \pm 0.48}$ & $64.70 \pm 1.81$ & $54.75 \pm 0.44$ & $47.68 \pm 0.66$ \\
& FS (\%) $\uparrow$ & $7.00 \pm 0.84$ & $42.30 \pm 0.94$ & $11.60 \pm 1.27$ & $\mathbf{44.50 \pm 1.01}$ & $7.60 \pm 1.64$ & $43.50 \pm 2.11$ & $8.40 \pm 1.71$ & $1.80 \pm 0.34$ & $1.00 \pm 0.42$ \\
& Return $\uparrow$ & $2.175 \pm 0.062$ & $3.069 \pm 0.023$ & $2.360 \pm 0.066$ & $3.092 \pm 0.016$ & $2.122 \pm 0.086$ & $\mathbf{3.095 \pm 0.020}$ & $2.193 \pm 0.087$ & $1.636 \pm 0.021$ & $1.331 \pm 0.045$ \\
& FE $\downarrow$ & $1.316 \pm 0.054$ & $0.590 \pm 0.017$ & $1.159 \pm 0.061$ & $0.567 \pm 0.018$ & $1.368 \pm 0.073$ & $\mathbf{0.566 \pm 0.011}$ & $1.256 \pm 0.100$ & $1.770 \pm 0.013$ & $2.007 \pm 0.039$ \\
\cmidrule(lr){1-11}
\multirow{4}{*}{Overlap-50}
& TC (\%) $\uparrow$ & $70.35 \pm 1.28$ & $\mathbf{84.55 \pm 0.69}$ & $68.05 \pm 4.14$ & $84.45 \pm 0.33$ & $67.50 \pm 2.12$ & $84.43 \pm 0.56$ & $72.88 \pm 1.05$ & $58.88 \pm 0.77$ & $47.05 \pm 0.38$ \\
& FS (\%) $\uparrow$ & $13.40 \pm 1.51$ & $\mathbf{45.90 \pm 1.92}$ & $12.40 \pm 3.67$ & $45.20 \pm 0.77$ & $11.70 \pm 3.14$ & $45.10 \pm 1.78$ & $17.50 \pm 1.31$ & $4.90 \pm 0.33$ & $0.80 \pm 0.49$ \\
& Return $\uparrow$ & $2.437 \pm 0.056$ & $\mathbf{3.124 \pm 0.034}$ & $2.344 \pm 0.184$ & $3.122 \pm 0.017$ & $2.327 \pm 0.104$ & $3.120 \pm 0.024$ & $2.566 \pm 0.051$ & $1.836 \pm 0.039$ & $1.340 \pm 0.015$ \\
& FE $\downarrow$ & $1.118 \pm 0.056$ & $0.587 \pm 0.026$ & $1.209 \pm 0.156$ & $\mathbf{0.574 \pm 0.015}$ & $1.141 \pm 0.124$ & $\mathbf{0.574 \pm 0.011}$ & $1.002 \pm 0.037$ & $1.629 \pm 0.035$ & $2.054 \pm 0.020$ \\
\cmidrule(lr){1-11}
\multirow{4}{*}{Dynamic}
& TC (\%) $\uparrow$ & $66.60 \pm 1.95$ & $81.30 \pm 0.34$ & $62.55 \pm 2.83$ & $\mathbf{82.50 \pm 0.43}$ & $63.23 \pm 0.83$ & $82.20 \pm 0.16$ & $65.90 \pm 2.33$ & $56.30 \pm 0.49$ & $47.55 \pm 0.40$ \\
& FS (\%) $\uparrow$ & $11.60 \pm 2.43$ & $39.30 \pm 0.64$ & $8.50 \pm 2.32$ & $\mathbf{42.80 \pm 0.72}$ & $7.90 \pm 0.93$ & $41.20 \pm 0.78$ & $11.10 \pm 1.87$ & $3.80 \pm 0.64$ & $0.60 \pm 0.40$ \\
& Return $\uparrow$ & $2.282 \pm 0.084$ & $2.995 \pm 0.014$ & $2.090 \pm 0.124$ & $\mathbf{3.048 \pm 0.018}$ & $2.162 \pm 0.041$ & $3.031 \pm 0.006$ & $2.269 \pm 0.098$ & $1.713 \pm 0.024$ & $1.353 \pm 0.035$ \\
& FE $\downarrow$ & $1.122 \pm 0.077$ & $0.590 \pm 0.004$ & $1.351 \pm 0.114$ & $\mathbf{0.564 \pm 0.014}$ & $1.207 \pm 0.046$ & $0.586 \pm 0.015$ & $1.147 \pm 0.086$ & $1.705 \pm 0.025$ & $2.054 \pm 0.031$ \\
\cmidrule(lr){1-11}
\multirow{4}{*}{Cascaded}
& TC (\%) $\uparrow$ & $58.95 \pm 1.70$ & $85.40 \pm 0.42$ & $65.55 \pm 1.76$ & $\mathbf{85.98 \pm 0.44}$ & $66.03 \pm 2.06$ & $84.90 \pm 0.51$ & $64.18 \pm 3.98$ & $54.90 \pm 1.08$ & $47.03 \pm 0.55$ \\
& FS (\%) $\uparrow$ & $4.10 \pm 0.87$ & $46.00 \pm 1.17$ & $8.60 \pm 1.76$ & $\mathbf{48.30 \pm 1.54}$ & $9.50 \pm 2.13$ & $44.80 \pm 1.85$ & $9.40 \pm 2.50$ & $2.70 \pm 0.37$ & $0.60 \pm 0.19$ \\
& Return $\uparrow$ & $1.936 \pm 0.076$ & $3.153 \pm 0.021$ & $2.194 \pm 0.078$ & $\mathbf{3.182 \pm 0.021}$ & $2.226 \pm 0.093$ & $3.126 \pm 0.026$ & $2.150 \pm 0.167$ & $1.625 \pm 0.054$ & $1.339 \pm 0.035$ \\
& FE $\downarrow$ & $1.539 \pm 0.115$ & $0.573 \pm 0.017$ & $1.368 \pm 0.070$ & $\mathbf{0.547 \pm 0.014}$ & $1.317 \pm 0.077$ & $0.599 \pm 0.021$ & $1.331 \pm 0.130$ & $1.804 \pm 0.043$ & $2.119 \pm 0.022$ \\
\bottomrule
\end{tabular}
\end{adjustbox}%
}
\end{table}
\FloatBarrier

\paragraph{Observation.}
The greatest improvements appear in RWARE-Advanced, where downstream decisions depend more strongly on the available task-relevant evidence. SPL increases TC, FS, and return and reduces FE across every matched backbone and participation pattern.

\subsection{Additional Isaac Lab Metrics}
\label{app:detailed-isaac}

Table~\ref{tab:appendix-isaac-full} reports the complete Isaac Lab evaluation,
including task completion (TC), full success (FS), episode return, failed
execution (FE), map coverage, task discovery, and wall-penetration steps. These
additional diagnostics complement the TC results reported in the main paper and
provide a broader view of the behaviors induced by SPL-MAPPO under embodied
UAV--UGV execution.

\begin{table}[!htbp]
\centering
\caption{Isaac Lab evaluation metrics. Values are mean $\pm$ standard error over three independent training seeds and 200 final evaluation episodes per seed.}
\label{tab:appendix-isaac-full}
\small
\setlength{\tabcolsep}{5.0pt}
\renewcommand{\arraystretch}{1.08}

\begin{adjustbox}{max width=\linewidth}
\begin{tabular}{llcQcQ}
\toprule
\multirow{2}{*}{Agents}
& \multirow{2}{*}{Metric}
& \multicolumn{2}{c}{Non-overlap}
& \multicolumn{2}{c}{Dynamic} \\
\cmidrule(lr){3-4}
\cmidrule(lr){5-6}
&
& MAPPO & SPL-MAPPO
& MAPPO & SPL-MAPPO \\
\midrule

\multirow{7}{*}{2}
& TC (\%) $\uparrow$
& $31.50 \pm 5.97$
& $\mathbf{61.58 \pm 0.67}$
& $36.67 \pm 0.71$
& $\mathbf{59.58 \pm 1.79}$ \\

& FS (\%) $\uparrow$
& $5.17 \pm 3.03$
& $\mathbf{35.17 \pm 0.73}$
& $9.00 \pm 0.76$
& $\mathbf{34.50 \pm 2.29}$ \\

& Return $\uparrow$
& $0.648 \pm 0.147$
& $\mathbf{1.540 \pm 0.016}$
& $0.785 \pm 0.022$
& $\mathbf{1.494 \pm 0.059}$ \\

& FE $\downarrow$
& $\mathbf{0.632 \pm 0.057}$
& $0.822 \pm 0.009$
& $0.798 \pm 0.003$
& $\mathbf{0.760 \pm 0.038}$ \\

& Map coverage (\%) $\uparrow$
& $36.21 \pm 1.69$
& $\mathbf{44.03 \pm 0.17}$
& $33.21 \pm 1.60$
& $\mathbf{44.03 \pm 1.07}$ \\

& Tasks discovered $\uparrow$
& $0.238 \pm 0.009$
& $\mathbf{0.287 \pm 0.003}$
& $0.150 \pm 0.020$
& $\mathbf{0.261 \pm 0.016}$ \\

& Wall-penetration steps $\downarrow$
& $0.000 \pm 0.000$
& $0.000 \pm 0.000$
& $0.000 \pm 0.000$
& $0.000 \pm 0.000$ \\

\midrule

\multirow{7}{*}{4}
& TC (\%) $\uparrow$
& $55.79 \pm 3.23$
& $\mathbf{76.38 \pm 1.52}$
& $58.46 \pm 3.40$
& $\mathbf{79.25 \pm 1.30}$ \\

& FS (\%) $\uparrow$
& $11.83 \pm 4.69$
& $\mathbf{39.83 \pm 3.22}$
& $12.17 \pm 4.21$
& $\mathbf{44.83 \pm 1.17}$ \\

& Return $\uparrow$
& $2.235 \pm 0.172$
& $\mathbf{3.330 \pm 0.092}$
& $2.343 \pm 0.174$
& $\mathbf{3.499 \pm 0.064}$ \\

& FE $\downarrow$
& $\mathbf{1.100 \pm 0.016}$
& $1.167 \pm 0.027$
& $\mathbf{1.160 \pm 0.023}$
& $1.193 \pm 0.020$ \\

& Map coverage (\%) $\uparrow$
& $\mathbf{67.96 \pm 1.45}$
& $54.90 \pm 0.36$
& $\mathbf{69.25 \pm 1.59}$
& $53.84 \pm 0.76$ \\

& Tasks discovered $\uparrow$
& $\mathbf{0.470 \pm 0.022}$
& $0.200 \pm 0.007$
& $\mathbf{0.512 \pm 0.041}$
& $0.190 \pm 0.009$ \\

& Wall-penetration steps $\downarrow$
& $0.000 \pm 0.000$
& $0.000 \pm 0.000$
& $0.000 \pm 0.000$
& $0.000 \pm 0.000$ \\

\bottomrule
\end{tabular}
\end{adjustbox}
\end{table}

\paragraph{Task performance.}
SPL-MAPPO achieves higher TC, FS, and return in all four
team-size--participation settings. For the two-agent system, TC increases from
$31.50\%$ to $61.58\%$ under Non-overlap and from $36.67\%$ to $59.58\%$
under Dynamic participation. The corresponding FS gains are also substantial,
increasing from $5.17\%$ to $35.17\%$ and from $9.00\%$ to $34.50\%$.
The same pattern persists with four agents: TC increases from $55.79\%$ to
$76.38\%$ under Non-overlap and from $58.46\%$ to $79.25\%$ under Dynamic
participation, while FS increases from $11.83\%$ to $39.83\%$ and from
$12.17\%$ to $44.83\%$. Episode return follows the same direction in all four
settings.

\paragraph{Execution and exploration diagnostics.}
FE does not change uniformly with SPL. In the two-agent setting it is higher
under Non-overlap and slightly lower under Dynamic participation, while in the
four-agent setting it is slightly higher under both schedules. These absolute
FE values occur under substantially different exploration and task-discovery
behavior and should therefore be interpreted together with those diagnostics.

The exploration pattern also differs with team size. With two agents,
SPL-MAPPO attains greater map coverage and discovers more tasks under both
participation schedules. For example, under Dynamic participation, map coverage
increases from $33.21\%$ to $44.03\%$, while task discovery increases from
$0.150$ to $0.261$. This pattern is consistent with broader evidence
acquisition. With four agents, SPL-MAPPO exhibits the opposite exploration
pattern: map coverage decreases from $67.96\%$ to $54.90\%$ under Non-overlap
and from $69.25\%$ to $53.84\%$ under Dynamic participation, while task
discovery also decreases. Despite this reduced exploration, SPL-MAPPO achieves
higher TC, FS, and return in both four-agent settings. This is consistent with
more selective acquisition and/or more effective use of the available
task-relevant evidence, although the reported diagnostics do not separate these
effects causally.

No wall-penetration steps are observed for either method in any of the four
settings, indicating that the performance differences reported above are not
accompanied by violations of this physical-safety diagnostic.

\subsection{Prior-Misspecification Robustness}
\label{app:prior-misspecification}

SPL uses task-family validity priors as part of the structured model for prospective acquisition supervision. We therefore evaluate how sensitive SPL is to misspecification of these priors. The true environment distribution is kept unchanged, and only the validity priors assumed by the acquisition model are perturbed. Let
\[
\mathbf{p}
=
(0.80,0.80,0.65,0.65,0.35,0.35,0.20,0.20)
\]
denote the nominal prior vector. For misspecification level
$\lambda\in[0,1]$, the acquisition model instead uses
\begin{equation}
\widehat{\mathbf{p}}_{\lambda}
=
(1-\lambda)\mathbf{p}
+
\lambda(1-\mathbf{p}).
\label{eq:app-prior-misspecification}
\end{equation}
Thus, $\lambda=0$ recovers the nominal priors,
$\lambda=0.5$ produces an uninformative prior of $0.5$ for every task,
and $\lambda=1$ fully reverses the nominal prior ordering. All other
structured information, including the fixed-valid-task constraint,
probe observation model, task utilities, environment dynamics, and
receiver learning procedure, remains unchanged.

We evaluate SPL-MAPPO under Non-overlap participation in
MPE-Advanced and RWARE-Advanced. Each misspecified condition uses
five independent training seeds, two million environment steps, and
200 final evaluation episodes per seed, matching the primary
four-agent protocol. The $\lambda=0$ values correspond to the nominal
SPL-MAPPO results reported in the main experiments; the
$\lambda\in\{0.25,0.50,1.00\}$ conditions comprise 30 additional
training runs. In all misspecified runs, the environment continues to
sample tasks from the nominal prior distribution in $\mathbf{p}$.

\begin{table}[!htbp]
\centering
\caption{Sensitivity of SPL-MAPPO to misspecified validity priors under
Non-overlap participation. Values are mean $\pm$ standard error over five
training seeds and 200 final evaluation episodes per seed. The environment
always uses the nominal prior distribution; only the acquisition model uses
$\widehat{\mathbf{p}}_{\lambda}$.}
\label{tab:app-prior-misspecification}
\small
\setlength{\tabcolsep}{4.0pt}
\renewcommand{\arraystretch}{1.05}

\begin{adjustbox}{max width=\linewidth}
\begin{tabular}{llQQQQ}
\toprule
Environment
& $\lambda$
& TC (\%) $\uparrow$
& FS (\%) $\uparrow$
& Return $\uparrow$
& FE $\downarrow$ \\
\midrule

\multirow{4}{*}{MPE-Advanced}
& $0.00$
& $78.23 \pm 0.16$
& $31.90 \pm 0.62$
& $27.390 \pm 0.106$
& $0.579 \pm 0.007$ \\

& $0.25$
& $77.70 \pm 0.39$
& $30.70 \pm 0.80$
& $27.02 \pm 0.15$
& $0.560 \pm 0.006$ \\

& $0.50$
& $76.43 \pm 0.11$
& $29.90 \pm 0.33$
& $26.70 \pm 0.10$
& $0.557 \pm 0.009$ \\

& $1.00$
& $76.73 \pm 0.21$
& $29.80 \pm 0.72$
& $26.59 \pm 0.05$
& $0.612 \pm 0.007$ \\

\midrule

\multirow{4}{*}{RWARE-Advanced}
& $0.00$
& $80.03 \pm 0.23$
& $42.20 \pm 1.44$
& $2.951 \pm 0.010$
& $0.583 \pm 0.011$ \\

& $0.25$
& $79.93 \pm 0.68$
& $39.70 \pm 1.59$
& $2.947 \pm 0.028$
& $0.586 \pm 0.011$ \\

& $0.50$
& $78.35 \pm 0.94$
& $33.90 \pm 1.43$
& $2.888 \pm 0.040$
& $0.666 \pm 0.023$ \\

& $1.00$
& $77.45 \pm 0.56$
& $35.50 \pm 1.15$
& $2.826 \pm 0.023$
& $0.652 \pm 0.007$ \\

\bottomrule
\end{tabular}
\end{adjustbox}
\end{table}

\paragraph{Observation.}
Prior misspecification reduces performance relative to the nominal SPL
configuration, but does not cause a collapse in task completion in either
tested environment. With fully reversed priors ($\lambda=1$), SPL-MAPPO
achieves $76.73\%$ TC in MPE-Advanced and $77.45\%$ in
RWARE-Advanced, compared with nominal values of $78.23\%$ and
$80.03\%$, respectively. These values also remain above the corresponding
Vanilla MAPPO means of $64.08\%$ and $58.63\%$ reported in the main
experiments.

The effect is not strictly monotonic in MPE-Advanced, where TC at
$\lambda=1$ is slightly higher than at $\lambda=0.5$. RWARE-Advanced
shows a clearer reduction in TC as the assumed prior becomes less accurate,
while full success is more sensitive than TC: FS decreases from
$42.20\%$ in the nominal setting to $33.90\%$ at $\lambda=0.5$ and
$35.50\%$ at $\lambda=1$. These results indicate that SPL retains
substantial task-completion performance under validity-prior
misspecification in the tested settings, while stricter episode-level success
can be more sensitive.

This experiment perturbs only the validity priors used by the acquisition
model. The fixed-valid-task constraint, probe accuracy, task utilities, and
other elements of the structured task model remain correctly specified.
Accordingly, the results provide evidence of tolerance to validity-prior
misspecification in the tested settings. Robustness to arbitrary errors in
the full structured task model is not established.

\subsection{Statistical Analysis of Primary Comparisons}
\label{app:statistical-analysis}

We conduct a retrospective statistical analysis of the 60 primary
Vanilla--SPL comparisons reported in the main paper. Each trained policy is
summarized by its mean task completion (TC) over 200 final evaluation episodes,
and independent training seeds are treated as the statistical replication
units. Throughout this subsection, $\Delta\mathrm{TC}$ denotes the absolute
difference between percentage-valued SPL and Vanilla TC means,
$\mathrm{TC}_{\mathrm{SPL}}-\mathrm{TC}_{\mathrm{Vanilla}}$, and is reported
using \%.

Same-training-seed pairing is available for 57 of the 60 configurations,
providing five paired seed-level differences in each testable configuration.
For these comparisons, we report the mean paired difference, a two-sided paired
Student $t$ test with four degrees of freedom, and a pointwise 95\% $t$
confidence interval. The Non-overlap MAPPO comparisons in MPE-Advanced,
RWARE-Base, and RWARE-Advanced use different Vanilla and SPL training-seed
sets. Their observed five-seed method means remain available for descriptive
comparison, but no paired confidence interval or paired $p$-value is reported.

To account for multiplicity, we apply Holm correction over the declared family
of all 60 primary comparisons at familywise level $0.05$. The three
configurations without verified seed pairing are assigned $p=1$ internally for
the correction and are displayed as unavailable below. This preserves the
original 60-comparison family and yields a conservative adjustment. Pointwise
confidence intervals are not multiplicity-adjusted.

\begin{table}[!htbp]
\centering
\caption{Retrospective statistical analysis for MPE-Base. Each training seed
is a replication unit. $\Delta$TC is the absolute SPL-minus-Vanilla difference
in percentage-valued TC scores. Confidence intervals are pointwise 95\%
intervals from five paired training seeds. Holm $p$ values are adjusted over
the declared family of 60 primary comparisons.}
\label{tab:spl-primary-mpe-base}
\small
\setlength{\tabcolsep}{5pt}
\renewcommand{\arraystretch}{1.03}
\begin{adjustbox}{max width=\linewidth}
\begin{tabular}{llrrrr}
\toprule
Participation & Backbone
& $\Delta$TC (\%)
& 95\% CI (\%)
& Raw $p$
& Holm $p$ \\
\midrule

Non-overlap
& IPPO
& 7.88
& [6.82, 8.93]
& $3.17\times10^{-5}$
& 0.0017 \\

& MAPPO
& 8.23
& [7.03, 9.42]
& $4.48\times10^{-5}$
& 0.0023 \\

& HAPPO
& 11.73
& [3.01, 20.44]
& 0.0202
& 0.3028 \\

\addlinespace[2pt]

Overlap-25
& IPPO
& 11.10
& [2.11, 20.09]
& 0.0266
& 0.3073 \\

& MAPPO
& 9.48
& [8.74, 10.21]
& $3.56\times10^{-6}$
& $2.10\times10^{-4}$ \\

& HAPPO
& 9.65
& [8.86, 10.44]
& $4.61\times10^{-6}$
& $2.67\times10^{-4}$ \\

\addlinespace[2pt]

Overlap-50
& IPPO
& 11.05
& [2.81, 19.29]
& 0.0204
& 0.3028 \\

& MAPPO
& 9.38
& [5.41, 13.34]
& 0.0028
& 0.0752 \\

& HAPPO
& 8.70
& [6.28, 11.12]
& $5.65\times10^{-4}$
& 0.0237 \\

\addlinespace[2pt]

Dynamic
& IPPO
& 13.45
& [5.36, 21.54]
& 0.0099
& 0.2081 \\

& MAPPO
& 8.65
& [7.41, 9.89]
& $4.24\times10^{-5}$
& 0.0022 \\

& HAPPO
& 10.10
& [6.97, 13.23]
& $8.59\times10^{-4}$
& 0.0326 \\

\addlinespace[2pt]

Cascaded
& IPPO
& 12.08
& [2.41, 21.74]
& 0.0256
& 0.3073 \\

& MAPPO
& 8.30
& [7.06, 9.54]
& $5.01\times10^{-5}$
& 0.0025 \\

& HAPPO
& 8.10
& [7.25, 8.95]
& $1.23\times10^{-5}$
& $7.01\times10^{-4}$ \\

\bottomrule
\end{tabular}
\end{adjustbox}
\end{table}

\begin{table}[!htbp]
\centering
\caption{Retrospective statistical analysis for MPE-Advanced. Notation and
statistical procedure follow Table~\ref{tab:spl-primary-mpe-base}. The
Non-overlap MAPPO comparison uses different Vanilla and SPL training-seed sets;
its $\Delta$TC is therefore descriptive and paired inference is unavailable.}
\label{tab:spl-primary-mpe-advanced}
\small
\setlength{\tabcolsep}{5pt}
\renewcommand{\arraystretch}{1.03}
\begin{adjustbox}{max width=\linewidth}
\begin{tabular}{llrrrr}
\toprule
Participation & Backbone
& $\Delta$TC (\%)
& 95\% CI (\%)
& Raw $p$
& Holm $p$ \\
\midrule

Non-overlap
& IPPO
& 7.58
& [7.05, 8.10]
& $2.26\times10^{-6}$
& $1.36\times10^{-4}$ \\

& MAPPO
& 14.15
& --
& --
& -- \\

& HAPPO
& 10.75
& [2.06, 19.44]
& 0.0264
& 0.3073 \\

\addlinespace[2pt]

Overlap-25
& IPPO
& 14.30
& [5.25, 23.35]
& 0.0118
& 0.2358 \\

& MAPPO
& 11.10
& [7.48, 14.72]
& 0.0010
& 0.0378 \\

& HAPPO
& 8.63
& [6.13, 11.12]
& $6.59\times10^{-4}$
& 0.0264 \\

\addlinespace[2pt]

Overlap-50
& IPPO
& 7.33
& [6.43, 8.22]
& $2.23\times10^{-5}$
& 0.0012 \\

& MAPPO
& 9.75
& [1.26, 18.24]
& 0.0332
& 0.3073 \\

& HAPPO
& 7.28
& [6.31, 8.24]
& $3.09\times10^{-5}$
& 0.0017 \\

\addlinespace[2pt]

Dynamic
& IPPO
& 11.60
& [0.73, 22.47]
& 0.0414
& 0.3073 \\

& MAPPO
& 12.03
& [3.78, 20.27]
& 0.0155
& 0.2784 \\

& HAPPO
& 17.08
& [1.68, 32.47]
& 0.0370
& 0.3073 \\

\addlinespace[2pt]

Cascaded
& IPPO
& 9.13
& [1.73, 16.52]
& 0.0266
& 0.3073 \\

& MAPPO
& 9.73
& [0.98, 18.47]
& 0.0366
& 0.3073 \\

& HAPPO
& 7.95
& [5.10, 10.80]
& 0.0015
& 0.0492 \\

\bottomrule
\end{tabular}
\end{adjustbox}

\par\smallskip
{\footnotesize
-- indicates that paired inference is unavailable because Vanilla and SPL use
different training-seed sets. The corresponding untestable comparison is
assigned $p=1$ internally for the 60-comparison Holm correction.}
\end{table}

\begin{table}[!htbp]
\centering
\caption{Retrospective statistical analysis for RWARE-Base. Notation and
statistical procedure follow Table~\ref{tab:spl-primary-mpe-base}. The
Non-overlap MAPPO comparison uses different Vanilla and SPL training-seed sets;
its $\Delta$TC is descriptive and paired inference is unavailable.}
\label{tab:spl-primary-rware-base}
\small
\setlength{\tabcolsep}{5pt}
\renewcommand{\arraystretch}{1.03}
\begin{adjustbox}{max width=\linewidth}
\begin{tabular}{llrrrr}
\toprule
Participation & Backbone
& $\Delta$TC (\%)
& 95\% CI (\%)
& Raw $p$
& Holm $p$ \\
\midrule

Non-overlap
& IPPO
& 20.38
& [4.50, 36.25]
& 0.0235
& 0.3056 \\

& MAPPO
& 15.33
& --
& --
& -- \\

& HAPPO
& 13.75
& [9.92, 17.58]
& $5.67\times10^{-4}$
& 0.0237 \\

\addlinespace[2pt]

Overlap-25
& IPPO
& 8.05
& [-1.30, 17.40]
& 0.0751
& 0.3073 \\

& MAPPO
& 14.83
& [6.33, 23.32]
& 0.0084
& 0.1842 \\

& HAPPO
& 18.20
& [9.60, 26.80]
& 0.0042
& 0.1049 \\

\addlinespace[2pt]

Overlap-50
& IPPO
& 7.00
& [3.73, 10.27]
& 0.0040
& 0.1043 \\

& MAPPO
& 13.85
& [4.26, 23.44]
& 0.0160
& 0.2784 \\

& HAPPO
& 18.53
& [12.18, 24.87]
& 0.0013
& 0.0440 \\

\addlinespace[2pt]

Dynamic
& IPPO
& 16.00
& [12.30, 19.70]
& $2.76\times10^{-4}$
& 0.0124 \\

& MAPPO
& 23.35
& [7.88, 38.82]
& 0.0138
& 0.2624 \\

& HAPPO
& 19.60
& [8.63, 30.57]
& 0.0077
& 0.1774 \\

\addlinespace[2pt]

Cascaded
& IPPO
& 19.45
& [13.18, 25.72]
& 0.0010
& 0.0371 \\

& MAPPO
& 19.23
& [12.63, 25.82]
& 0.0013
& 0.0440 \\

& HAPPO
& 19.38
& [11.51, 27.24]
& 0.0024
& 0.0693 \\

\bottomrule
\end{tabular}
\end{adjustbox}
\end{table}

\begin{table}[!htbp]
\centering
\caption{Retrospective statistical analysis for RWARE-Advanced. Notation and
statistical procedure follow Table~\ref{tab:spl-primary-mpe-base}. The
Non-overlap MAPPO comparison uses different Vanilla and SPL training-seed sets;
its $\Delta$TC is descriptive and paired inference is unavailable.}
\label{tab:spl-primary-rware-advanced}
\small
\setlength{\tabcolsep}{5pt}
\renewcommand{\arraystretch}{1.03}
\begin{adjustbox}{max width=\linewidth}
\begin{tabular}{llrrrr}
\toprule
Participation & Backbone
& $\Delta$TC (\%)
& 95\% CI (\%)
& Raw $p$
& Holm $p$ \\
\midrule

Non-overlap
& IPPO
& 27.85
& [23.14, 32.56]
& $8.06\times10^{-5}$
& 0.0040 \\

& MAPPO
& 21.40
& --
& --
& -- \\

& HAPPO
& 21.98
& [11.18, 32.77]
& 0.0048
& 0.1160 \\

\addlinespace[2pt]

Overlap-25
& IPPO
& 19.05
& [14.95, 23.15]
& $2.08\times10^{-4}$
& 0.0096 \\

& MAPPO
& 14.83
& [11.13, 18.52]
& $3.71\times10^{-4}$
& 0.0163 \\

& HAPPO
& 20.45
& [16.11, 24.79]
& $1.98\times10^{-4}$
& 0.0093 \\

\addlinespace[2pt]

Overlap-50
& IPPO
& 14.20
& [9.08, 19.32]
& 0.0015
& 0.0492 \\

& MAPPO
& 16.40
& [4.79, 28.01]
& 0.0172
& 0.2784 \\

& HAPPO
& 16.93
& [10.26, 23.59]
& 0.0021
& 0.0639 \\

\addlinespace[2pt]

Dynamic
& IPPO
& 14.70
& [8.73, 20.67]
& 0.0024
& 0.0693 \\

& MAPPO
& 19.95
& [12.23, 27.67]
& 0.0020
& 0.0620 \\

& HAPPO
& 18.98
& [16.46, 21.49]
& $3.05\times10^{-5}$
& 0.0017 \\

\addlinespace[2pt]

Cascaded
& IPPO
& 26.45
& [21.10, 31.80]
& $1.63\times10^{-4}$
& 0.0078 \\

& MAPPO
& 20.43
& [15.09, 25.76]
& $4.43\times10^{-4}$
& 0.0190 \\

& HAPPO
& 18.88
& [13.19, 24.56]
& $7.69\times10^{-4}$
& 0.0300 \\

\bottomrule
\end{tabular}
\end{adjustbox}
\end{table}

\FloatBarrier
\paragraph{Summary.}
SPL has a higher observed mean TC than its corresponding Vanilla backbone in
all 60 primary configurations. The equal-weight descriptive mean difference
across these configurations is $14.09\%$; the corresponding backbone-level
mean differences are $13.93\%$, $14.02\%$, and $14.33\%$ for IPPO, MAPPO,
and HAPPO, respectively. Among the 57 configurations with paired inference,
56 pointwise 95\% confidence intervals have a positive lower bound. The single
exception is RWARE-Base Overlap-25 with IPPO, whose pointwise interval is
$[-1.30,17.40]\%$. After Holm correction over the declared family of 60
comparisons, 29 configurations remain significant at familywise level $0.05$,
while 28 testable comparisons do not. The three remaining configurations have
no paired test because same-training-seed pairing is unavailable.

These results distinguish descriptive consistency from multiplicity-adjusted
statistical evidence. Higher observed SPL means occur in all 60 configurations,
but this does not imply that all 60 comparisons are individually statistically
significant. Similarly, failure to reject the null hypothesis after Holm
correction should not be interpreted as evidence of equivalence between SPL
and the corresponding Vanilla method.

\paragraph{Evaluation design and zero standard errors.}
Final evaluation uses deterministic policies and a fixed set of 200 reset seeds
within each environment protocol, reused across training seeds and methods.
This controls variation due to the sampled evaluation instances while retaining
independent policy training across training seeds.

Four TC cells in the primary results table have an exact seed-level standard
error of zero. These are Vanilla IPPO and Vanilla MAPPO under MPE-Base
Non-overlap, MAE-MAPPO under MPE-Base Non-overlap, and Vanilla IPPO under
MPE-Advanced Non-overlap. In each case, the five independently trained policies
produce exactly identical 200-episode mean TC values at full precision.
The reported zero therefore concerns variability among the five seed-level
evaluation means; it does not imply that individual evaluation episodes have
zero variability. The available records establish this numerical plateau under
the fixed deterministic evaluation schedule but do not establish a particular
behavioral mechanism responsible for it.

\paragraph{Limitations of the statistical analysis.}
This analysis is retrospective and was conducted on the existing completed
runs without additional training or evaluation. The paired tests treat the five
training seeds as independent replications and assume that the seed-level
differences are approximately normally distributed. With only five pairs per
configuration, this normality assumption cannot be assessed convincingly.
Furthermore, because the same fixed evaluation-instance schedule is reused
across training seeds and methods, the reported inference is conditional on
this evaluation design and does not quantify uncertainty over an unrestricted
distribution of deployment tasks.

The three Non-overlap MAPPO configurations without verified same-seed pairing
prevent a complete 60-comparison paired analysis. Assigning these unavailable
tests $p=1$ internally makes the reported Holm correction conservative while
preserving the declared family size, but does not replace the missing paired
evidence. Pointwise confidence intervals should also be interpreted separately
from the Holm-adjusted hypothesis tests.

\end{document}